\documentclass[twoside,11pt]{article}
\usepackage{jmlr2e}
\pdfoutput=1
\ShortHeadings{Causal Discovery with Continuous Additive Noise Models}{J.~Peters, J.M.~Mooij, D.~Janzing and B.~Sch\"olkopf}
\firstpageno{1} 

\usepackage{natbib}
\usepackage{amsmath}
\usepackage{amssymb}
\usepackage{tikz}
\usepackage{hyperref}
\usepackage{comment}
\usepackage{algorithm}
\usepackage{algorithmic}
\usepackage{xcolor,colortbl}
\usepackage{multirow}
\usepackage{hhline}
\usepackage{placeins} 

\newtheorem{problem}[theorem]{Problem}
\newtheorem{ex}[theorem]{Example}
\newtheorem{condition}[theorem]{Condition}

\usetikzlibrary{arrows}
\usepackage{paralist}
\newdimen\arrowsize
\pgfarrowsdeclare{arcsq}{arcsq}
{
  \arrowsize=0.2pt
  \advance\arrowsize by .5\pgflinewidth
  \pgfarrowsleftextend{-4\arrowsize-.5\pgflinewidth}
  \pgfarrowsrightextend{.5\pgflinewidth}
}
{
  \arrowsize=1.5pt
  \advance\arrowsize by .5\pgflinewidth
  \pgfsetdash{}{0pt} 
  \pgfsetroundjoin   
  \pgfsetroundcap    
  \pgfpathmoveto{\pgfpoint{0\arrowsize}{0\arrowsize}}
  \pgfpatharc{-90}{-140}{4\arrowsize}
  \pgfusepathqstroke
  \pgfpathmoveto{\pgfpointorigin}
  \pgfpatharc{90}{140}{4\arrowsize}
  \pgfusepathqstroke
}

\DeclareMathOperator*{\argmin}{argmin}
\DeclareMathOperator*{\argmax}{argmax}

\newcommand{\independent}{\mbox{${}\perp\mkern-11mu\perp{}$}}
\newcommand{\notindependent}{\mbox{${}\not\!\perp\mkern-11mu\perp{}$}}
\newcommand{\iid}{\overset{\text{iid}}{\sim}}

\newcommand{\C}[1]{\mathcal{#1}}
\newcommand{\B}[1]{\mathbf{#1}}

\newcommand{\R}{{\mathbb R}}
\newcommand{\RN}{{\mathbb R}}

\newcommand{\X}{{\mathbf X}}

\newcommand{\SE}{\C{S}}
\newcommand{\graph}{{\mathbf{graph}}}

\newcommand{\Gp}{\C{G}'}
\newcommand{\G}{\C{G}}
\newcommand{\Gc}{\C{G}_c}
\newcommand{\mean}{{\mathbf E}}  
\newcommand{\var}{{\mathbf{var}}}  
\newcommand{\eps}{{N}}  
\newcommand{\PA}[2][]{{\B{PA}}^{#1}_{#2}}

\newcommand{\CH}[2][]{{\B{CH}}^{#1}_{#2}}
\newcommand{\DE}[2][]{{\B{DE}}^{#1}_{#2}}
\newcommand{\ND}[2][]{{\B{ND}}^{#1}_{#2}}
\newcommand{\eref}[1]{(\ref{#1})}
\newcommand{\given}{\,|\,}

\newcommand{\dd}{{\partial }}
\newcommand{\dx}{{\partial x}}
\newcommand{\dy}{{\partial y}}

\newcommand{\law}[1]{\mathcal{L}({#1})}
\newcommand{\lawX}{{\law{\mathbf X}}}

\newcommand{\lawN}{{\law{\mathbf N}}}
\newcommand{\llaw}{{\mathbf{law}_X}}

\newcommand{\SEM}{SEM }
\newcommand{\DAG}{DAG }

\newcommand{\DAGs}{DAGs }
\newcommand{\tc}{\gamma}
\newcommand{\td}{\delta}
\newcommand{\tnu}{\tilde \nu}

\begin{document}

\title{Causal Discovery with Continuous Additive Noise Models}

\author{\name Jonas Peters\thanks{Part of this work was done while JP and JMM were with the MPI T\"ubingen.} 
\email peters@stat.math.ethz.ch\\
\addr Seminar for Statistics\\
ETH Z\"urich\\
Switzerland
\AND
\name Joris M.~Mooij$^*$ 
\email j.m.mooij@uva.nl\\
\addr Institute for Informatics\\
\addr University of Amsterdam\\
The Netherlands\\[0.5\baselineskip]
\addr Institute for Computing and Information Sciences\\
\addr Radboud University Nijmegen\\
The Netherlands
\AND
\name Dominik Janzing \email janzing@tuebingen.mpg.de\\
\name Bernhard Sch\"olkopf \email bs@tuebingen.mpg.de\\
\addr Max Planck Institute for Intelligent Systems\\
T\"ubingen\\
Germany
}

\editor{??}

\maketitle
\begin{abstract}%
We consider the problem of learning causal directed acyclic graphs from an observational joint distribution. 
One can use these graphs to predict the outcome of interventional experiments, from which data are often not available.
We show that if the observational distribution follows a structural equation model with an additive noise structure, the directed acyclic graph becomes identifiable from the distribution under mild conditions. This constitutes an interesting alternative to traditional methods that assume faithfulness and identify only the Markov equivalence class of the graph, thus leaving some edges undirected. We provide practical algorithms for finitely many samples, RESIT (Regression with Subsequent Independence Test) and two methods based on an independence score. We prove that RESIT is correct in the population setting and provide an empirical evaluation. 
\end{abstract}

\section{Introduction} \label{sec:intro}
Many scientific questions deal with the causal structure of a data-generating process. If we know the reasons why an individual is more susceptible to a disease than others, for example, we can hope to develop new drugs in order to cure this disease or prevent its outbreak. Recent results indicate that knowing the causal structure is also useful for classical machine learning tasks. In the two variable case, for example, knowing which is cause and which is effect has implications for semi-supervised learning and covariate shift adaptation \citep{ScholkopfJPSZMJ2012}.

We consider a $p$-dimensional random vector $\X = (X_1, \ldots, X_p)$ with a joint distribution $\lawX$ and assume that there is a true acyclic causal graph $\G$ that describes the data generating process (see Section~\ref{sec:tcg}).
In this work we address the following problem of causal inference: given the distribution $\lawX$
we try to infer the graph $\G$.
A priori, the causal graph contains information about the physical process that cannot be found in properties of the joint distribution.
One therefore requires assumptions connecting these two worlds. While traditional methods like PC, FCI \citep{Spirtes2000} or score-based approaches \citep[e.g.][]{Chickering2002}, that are explained in more detail in Section~\ref{sec:exi}, make assumptions that enable us to recover the graph up to the Markov equivalence class, we investigate a different set of assumptions. If the data have been generated by an additive noise model (see Section~\ref{sec:ide}), we will generically be able to recover the correct graph from the joint distribution.\\

In the remainder of this section we set up the required notation and definitions for graphs (Section~\ref{sec:cg}), briefly introduce Judea Pearl's do-calculus (Section~\ref{sec:do}) and use it to define our object of interest, a true causal graph (Section~\ref{sec:tcg}). We introduce structural equation models (SEMs) in Section~\ref{sec:sem}. After discussing existing methods in Section~\ref{sec:exi}, we provide the main results of this work in Section~\ref{sec:ide}. We prove that for additive noise models (ANMs), a special class of SEMs, one can identify the graph from the joint distribution. This is possible not only for additive noise models but for all classes of SEMs that are able to identify graphs from a bivariate distribution, meaning they can distinguish between cause and effect. Section~\ref{sec:alg} proposes and compares algorithms that can be used in practice, when instead of the joint distribution, we are only given i.i.d.\ samples. These algorithms are tested in Section~\ref{sec:exp}.\\

This paper builds on the conference papers of \citet{Hoyer2008}, \citet{Peters2011b} and \citet{Mooij2009}\footnote{Parts of Sections~\ref{sec:intro} and~\ref{sec:exi} have been taken and modified from the PhD thesis of \citet{PetersThesis}.} but extends the material in several aspects. All deliberations in Section~\ref{sec:tcg} about the true causal graph and Example~\ref{ex:cou} are novel.
The presentation of the theoretical results in Section~\ref{sec:ide} is improved.
In particular, we added the motivating Example~\ref{ex:motivating} and Propositions~\ref{prop:cme} and~\ref{prop:graph}.
Example~\ref{ex:cou2} provides a non-identifiable case different from the linear Gaussian example. 
Proposition~\ref{prop:kun} is based on \citep{Zhang2009} and contains important necessary conditions for the failure of identifiability.
In Corollary~\ref{cor:new} we present a novel identifiability result for a class of nonlinear functions and Gaussian noise variables. 
Proposition~\ref{prop:nonconst} proves that causal minimality is satisfied if the structural equations do not contain constant functions.
Section~\ref{sec:toporder} contains results that guarantee to find the set of correct topological orderings when the assumption of causal minimality is dropped. 
Theorem~\ref{thm:algo} proves a conjecture from \citet{Mooij2009} by showing that given an independence oracle the algorithm provided in \citet{Mooij2009} is correct.
We propose a new score function for estimating the true directed acyclic graph in Section~\ref{sec:indbscore} and present two corresponding score-based methods. 
We provide an extended section on simulation experiments and discuss experiments on real data.

\subsection{Directed Acyclic Graphs} \label{sec:cg}
We start with some basic notation for graphs.
Consider a finite family of random variables $\X = (X_1, \ldots, X_p)$ with index set $\B{V} := \{1, \ldots, p\}$ (we use capital letters for random variables and bold letters for sets and vectors). We denote their joint distribution by $\lawX$.
We write $p_{X_1}(x)$ or simply $p(x)$ for the Radon-Nikodym derivative of $\law{X_1}$ either with respect to the Lebesgue or the counting measure and (sometimes implicitly) assume its existence.
A graph $\G=(\B{V},\C{E})$ consists of nodes $\B{V}$ and edges $\C{E} \subseteq \B{V}^2$ with $(v,v) \not\in \C{E}$ for any $v\in \B{V}$. In a slight abuse of notation we identify the nodes (or vertices) $j \in \B{V}$ with the variables $X_j$, the context should clarify the meaning.
We also consider sets of variables $\B{S} \subseteq \B{X}$ as a single multivariate variable.
We now introduce graph terminology that we require later. Most of the definitions can be found in \citep{Spirtes2000, Koller2009,Lauritzen1996}, for example.

Let $\G=(\B{V},\C{E})$ be a graph with $\B{V} := \{1, \ldots, p\}$ and corresponding random variables $\X = (X_1, \ldots, X_p)$.
A graph $\G_1=(\B{V}_1,\C{E}_1)$ is called a {\bf subgraph} of $\G$ if $\B{V}_1 = \B{V}$ and $\C{E}_1 \subseteq \C{E}$; we then write $\G_1 \leq \G$. If additionally, $\C{E}_1 \neq \C{E}$, we call $\G_1$ a {\bf proper subgraph} of $\G$.  

A node $i$ is called a {\bf parent} of $j$ if $(i,j) \in \C{E}$ and a {\bf child} if $(j,i) \in \C{E}$. The set of parents of $j$ is denoted by $\PA[\G]{j}$, the set of its children by $\CH[\G]{j}$. Two nodes $i$ and $j$ are {\bf adjacent} if either $(i,j) \in \C{E}$ or $(j,i) \in \C{E}$.
We call $\G$ {\bf fully connected} if all pairs of nodes are adjacent. We say that there is an undirected edge between two adjacent nodes $i$ and $j$ if $(i,j) \in \C{E}$ and $(j,i) \in \C{E}$. An edge between two adjacent nodes is directed if it is not undirected. We then write $i \rightarrow j$ for $(i,j) \in \C{E}$.
Three nodes are called an {\bf immorality} or a {\bf v-structure} if one node is a child of the two others that themselves are not adjacent. 
The {\bf skeleton} of $\G$ is the set of all edges without taking the direction into account, that is all $(i,j)$, such that $(i,j) \in \C{E}$ or $(j,i) \in \C{E}$.

A {\bf path} in $\G$ is a sequence of (at least two) distinct vertices ${i_1}, \ldots, {i_n}$, such that there is an edge between ${i_k}$ and ${i_{k+1}}$ for all $k=1, \ldots, n-1$. If $({i_k},{i_{k+1}}) \in \C{E}$ and $({i_{k+1}},{i_{k}}) \notin \C{E}$ for all $k$ we speak of a {\bf directed path} between ${i_1}$ and ${i_n}$ and call ${i_n}$ a {\bf descendant} of ${i_1}$. We denote all descendants of ${i}$ by $\DE[\G]{i}$ and all non-descendants of ${i}$, excluding $i$, by $\ND[\G]{i}$. In this work, $i$ is neither a descendant nor a non-descendant of itself.
If $({i_{k-1}}, {i_{k}}) \in \C{E}$ and $({i_{k+1}}, {i_{k}}) \in \C{E}$, and also $(i_{k},i_{k-1}) \notin \C{E}$ and $(i_{k},i_{k+1}) \notin \C{E}$, ${i_k}$ is called a {\bf collider} on this path.
$\G$ is called a {\bf partially directed acyclic graph (PDAG)} if there is no directed cycle, i.e., no pair ($j$, $k$), such that there are directed paths from $j$ to $k$ and from $k$ to $j$. 
$\G$ is called a {\bf directed acyclic graph (DAG)} if it is a PDAG and all edges are directed.

In a DAG, a path between ${i_1}$ and ${i_n}$ is {\bf blocked by a set $\B{S}$} (with neither ${i_1}$ nor ${i_n}$ in this set) whenever there is a node ${i_k}$, such that one of the following two possibilities hold:
1. ${i_k} \in \B{S}$ and
${i_{k-1}} \rightarrow {i_k} \rightarrow {i_{k+1}}$ or
${i_{k-1}} \leftarrow {i_k} \leftarrow {i_{k+1}}$ or
${i_{k-1}} \leftarrow {i_k} \rightarrow {i_{k+1}}$
Or 2., ${i_{k-1}} \rightarrow {i_k} \leftarrow {i_{k+1}}$ and neither ${i_k}$ nor any of its descendants is in $\B{S}$.
We say that two disjoint subsets of vertices $\B{A}$ and $\B{B}$ are {\bf $d$-separated} by a third (also disjoint) subset $\B{S}$ if every path between nodes in $\B{A}$ and $\B{B}$ is blocked by $\B{S}$.
Throughout this work, $\independent$ denotes (conditional) independence. 
The joint distribution $\lawX$ is said to be {\bf Markov with respect to the DAG $\G$} if 
$$
\B{A}, \B{B}\; d\text{-sep. by } \B{C} \; \Rightarrow \; \B{A} \independent \B{B} \given \B{C}
$$
for all disjoint sets $\B{A},\B{B},\B{C}$.
$\lawX$ is said to be {\bf faithful to the DAG $\G$} if 
$$
\B{A}, \B{B}\; d\text{-sep. by } \B{C} \; \Leftarrow \; \B{A} \independent \B{B} \given \B{C}
$$
for all disjoint sets $\B{A},\B{B},\B{C}$. 
A distribution satisfies {\bf causal minimality} with respect to $\G$ if it is Markov with respect to $\G$, but not to any proper subgraph of $\G$.
We denote by $\mathcal{M}(\G)$ the set of distributions that are Markov with respect to $\G$:
$
\mathcal{M}(\G) := \{\lawX\,:\,\lawX \text{ is Markov w.r.t. }\G \}\,.
$
Two DAGs $\G_1$ and $\G_2$ are {\bf Markov equivalent} if $\mathcal{M}(\G_1) = \mathcal{M}(\G_2)$. This is the case if and only if $\G_1$ and $\G_2$ satisfy the same set of $d$-separations, that means the Markov condition entails the same set of (conditional) independence conditions. The set of all DAGs that are Markov equivalent to some DAG (a so-called Markov equivalence class) can be represented by a {\bf completed PDAG}. This graph satisfies $(i,j) \in \C{E}$ if and only if one member of the Markov equivalence class does.
\citet{Verma1991} showed that:
\begin{lemma} \label{lemma_immo}
Two DAGs are Markov equivalent if and only if they have the same skeleton and the same immoralities.
\end{lemma}


Faithfulness is not very intuitive at first glance. 
We now give an example of a distribution that is Markov but not faithful with respect to some DAG $\G_1$. This is achieved by making two paths cancel each other and creating an independence that is not implied by the graph structure.
\begin{ex} \label{ex:nf}
Consider the two graphs in Figure~\ref{fig:cmaf}.
\begin{figure}
\begin{center}
\begin{tikzpicture}[xscale=1.8,yscale=0.8,line width=0.5pt, inner sep=0.4mm, outer sep=0.5mm, shorten >=1pt, shorten <=1pt]
  \draw (0,2) node(x) [circle, draw] {$X$};
  \draw (0,0) node(z) [circle, draw] {$Z$};
  \draw (1,1) node(y) [circle, draw] {$Y$};
  \draw[-arcsq] (x) to node [left] {$c$} (z);
  \draw[-arcsq] (x) to node [above] {$a$} (y);
  \draw[-arcsq] (y) to node [below] {$b$} (z);
\end{tikzpicture}
\hspace{0.3\linewidth}
\begin{tikzpicture}[xscale=1.8,yscale=0.8,line width=0.5pt, inner sep=0.4mm, outer sep=0.5mm, shorten >=1pt, shorten <=1pt]
  \draw (0,2) node(x) [circle, draw] {$X$};
  \draw (0,0) node(z) [circle, draw] {$Z$};
  \draw (1,1) node(y) [circle, draw] {$Y$};
  \draw[-arcsq] (x) to node [above] {$\tilde a$} (y);
  \draw[-arcsq] (z) to node [below] {$\tilde b$} (y);
\end{tikzpicture}\\
$\quad \G_1$ \hspace{5.3cm} $\G_2$
\end{center}
\caption{After fine-tuning the parameters for the two graphs, both models generate the same joint distribution.}
\label{fig:cmaf}
\end{figure}
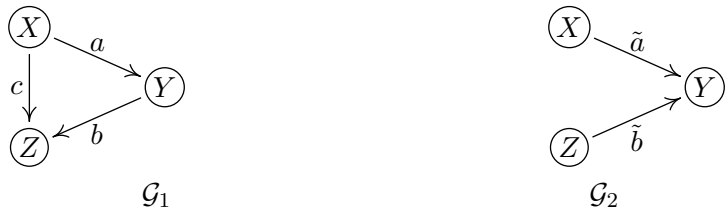
Corresponding to the left graph we generate a joint distribution by the following equations.
$X=N_X, Y=a X + N_Y, Z=b Y + c X +N_Z$,
with $N_X \sim \mathcal{N}(0,\sigma^2_X)$, $N_Y \sim \mathcal{N}(0,\sigma^2_Y)$ and $N_Z \sim \mathcal{N}(0,\sigma^2_Z)$ jointly independent.
This is an example of a linear Gaussian structural equation model with graph $\G_1$ that we formally define in Section~\ref{sec:sem}.
Now, if 
$
a\cdot b + c = 0
$,
the distribution is not faithful\footnote{More precisely: not triangle-faithful \citep{Zhang2008}.} with respect to $\G_1$ since we obtain $X \independent Z$.

Correspondingly, we generate a distribution related to graph $\G_2$:
$X=\tilde{N}_X,
Y=\tilde{a} X + \tilde{b} Z + \tilde{N}_Y,
Z=\tilde{N}_Z$, 
with all $\tilde{N}_{\cdot} \sim \mathcal{N}(0,\tau^2_{\cdot})$ jointly independent.
If we choose 
$\tau_X^2 = \sigma_X^2$, $\tilde a = a$, $\tau_Z^2 = b^2 \sigma_Y^2 + \sigma_Z^2$, 
$\tilde b = (b \sigma_Y^2)/(b^2 \sigma_Y^2 + \sigma_Z^2)$ and
$\tau_Y^2 = \sigma_Y^2 - (b^2 \sigma_Y^4)/(b^2 \sigma_Y^2 + \sigma_Z^2)$, 
both models lead to the covariance matrix
$$
\Sigma  = \left( 
\begin{array}{ccc}
\sigma_X^2 & a\sigma_X^2 & 0\\
a\sigma_X^2 & a^2\sigma_X^2 + \sigma_Y^2& b\sigma_Y^2\\
0 & b\sigma_Y^2 & b^2\sigma_Y^2 + \sigma_Z^2
\end{array}
\right)
$$
and thus to the same distribution. It can be checked that the distribution is faithful with respect to $\G_2$ if $\tilde{a},\tilde{b} \neq 0$ and all $\tilde{\tau_{\cdot}} > 0$.
\end{ex}
The distribution from Example~\ref{ex:nf} is faithful with respect to $\G_2$, but not with respect to $\G_1$. Nevertheless, for both models, causal minimality is satisfied if none of the parameters vanishes: the distribution is not Markov to any proper subgraph of $\G_1$ or $\G_2$ since removing an arrow would correspond to a new (conditional) independence that does not hold in the distribution. Note that $\G_2$ is not a proper subgraph of $\G_1$. In general, causal minimality is weaker than faithfulness:
\begin{remark} \label{rem:cmiw}
If $\lawX$ is faithful with respect to $\G$, then causal minimality is satisfied.
\end{remark}
This is due to the fact that any two nodes that are not directly connected by an edge can be $d$-separated.
Another, equivalent formulation of causal minimality reads as follows:
\begin{proposition} \label{prop:cme}
  Consider the random vector $\X = (X_1, \ldots, X_p)$ and assume that the joint distribution has a density with respect to a product measure. Suppose that $\lawX$ is Markov with respect to $\G$. Then 
$\lawX$ satisfies causal minimality with respect to $\G$ if and only if $\forall X_j \; \forall Y \in \PA[\G]{j} \,$ we have that $X_j \notindependent Y\,|\,\PA[\G]{j} \setminus \{Y\}$.
\end{proposition}
\begin{proof}
See Appendix~\ref{app:proofpropcme}.
\end{proof}

\subsection{Interventional Distributions} \label{sec:do}
Given a directed acyclic graph (DAG)
$\G$, \citet{Pearl2009} introduces the $do$-notation as a mathematical description of interventional experiments. More precisely, $do (X_j = \tilde p(x_j))$ stands for setting the variable $X_j$ randomly according to the distribution $\tilde p(x_j)$, irrespective of its parents, while not interfering with any other variable. Formally: 
\begin{definition} \label{def:do}
Let $\X=(X_1, \ldots, X_p)$ be a collection of variables
with joint distribution $\lawX$ that we assume to be absolutely continuous with respect to the Lebesgue measure or the counting measure (i.e., there exists a probability density function or a probability mass function).
Given a DAG $\G$ over $\X$, we define the {\it interventional distribution $do (X_j= \tilde p(x_j))$} of $X_1, \ldots, X_p$ by
$$
p\big(x_1, \ldots, x_p \, | \, do (X_j=\tilde p(x_j))\big) := \prod_{i\neq j}^p p(x_i|x_{\PA[]{i}})\cdot \tilde p(x_j)\,,
$$
if $p(x_1, \ldots, x_p)>0$ and zero otherwise.
Here $\tilde p(x_j)$ is either a probability density function or a probability mass function.
Similarly, we can intervene at different nodes at the same time by defining the interventional distribution $do (X_j = \tilde p(x_j), j \in \B{J})$ for $\B{J} \subseteq \B{V}$ as
\begin{equation*}
p\big(x_1, \ldots, x_p \, | \, do (X_j=\tilde p(x_j), j \in \B{J})\big) := \prod_{i\notin \B{J}} p(x_i|x_{\PA[]{i}})\cdot \prod_{j\in \B{J}}\tilde p(x_j)
\end{equation*}
if $p(x_1, \ldots, x_p)>0$ and zero otherwise.
\end{definition}
Here, $x_{\PA[]{i}}$ denotes the tuple of all $x_j$ for $X_j$ being a parent of $X_i$ in $\G$.
\citet{Pearl2009} introduces Definition~\ref{def:do} with the special case of $\tilde p(x_j) = \delta_{x_j , \tilde x_j}$, 
where $\delta_{x_j , \tilde x_j}=1$ if $x_j=\tilde x_j$ and $\delta_{x_j , \tilde x_j}=0$ otherwise; this corresponds to a point mass at $\tilde{x_j}$. For more details on {\it soft} interventions, see \citet{Eberhardt2007}.
Note that in general: 
$$
p(x_1, \ldots, x_p \, | \, do (X_j=\tilde x_j)) \neq p(x_1, \ldots, x_p \,|\, X_j=\tilde x_j)\,.
$$
The expression $p\big(x_1, \ldots, x_p \, | \, do (X_j=\tilde x_j, j \in \B{J})\big)$ yields a distribution over $X_1, \ldots, X_p$. If we are only interested in computing the marginal $p\big(x_i \, | \, do (X_j=\tilde x_j))$, where $X_i$ is not a parent of $X_j$, we can use the parent adjustment formula \citep[][Theorem~3.2.2]{Pearl2009}
\begin{equation} \label{eq:parentadj}
p(x_i \, | \,  do (X_j=\tilde x_j)) = \sum_{x_{\PA[]{j}}} p(x_i \given \tilde x_j, x_{\PA[]{j}}) \, p(x_{\PA[]{j}}) \,.
\end{equation}

\subsection{True Causal Graphs} \label{sec:tcg}
In this section we clarify what we mean by a true causal graph $\Gc$. In short, we use this term if one can read off the results of randomized studies from $\Gc$ and the observational joint distribution. This means that the graph and the observational joint distribution lead to causal effects that one observes in practice. Two important restrictive assumptions that we make throughout this work are \emph{acyclicity} (the absence of directed cycles, in other words, no causal feedback loops are allowed) and \emph{causal sufficiency} (the absence of hidden variables that are a common cause of at least two observed variables).
\begin{definition} \label{def:tcg}
Assume we are given a distribution $\lawX$ over $X_1, \ldots, X_p$ 
and distributions 
$\C{L}_{do (X_j = \tilde p(x_j), j \in \B{J})}(\B{X})$ for all $\B{J} \subseteq \B{V} = \{1, \ldots, p\}$ (think of the variables $X_j$ having been randomized).
We then call the graph $\Gc$ a \emph{true causal graph} for these distributions if 
\begin{compactitem}
\item $\Gc$ is a directed acyclic graph;
\item the distribution $\lawX$ is Markov with respect to $\Gc$;
\item for all $\B{J} \subseteq \B{V}$ and $\tilde p(x_j)$ with $j \in \B{J}$ the distribution 
$\C{L}_{do (X_j = \tilde p(x_j), j \in \B{J})}(\B{X})$
coincides with $p\big(x_1, \ldots, x_p \,|\, do (X_j=\tilde p(x_j)), j \in \B{J}\big)$, computed from $\Gc$ as in Definition~\ref{def:do}.
\end{compactitem}
\end{definition}
Definition~\ref{def:tcg} is purely mathematical if one considers $\C{L}_{do (X_j = \tilde p(x_j), j \in \B{J})}(\B{X})$ as an abstract family of given distributions. But it is a small step to make the relation to the ``real world''. We call $\Gc$ the true causal graph {\it of a data generating process} if it is the true causal graph for the distributions $\lawX$ and $\C{L}_{do (X_j = \tilde p(x_j), j \in \B{J})}(\B{X})$, where the latter are obtained by randomizing $X_j$ according to $\tilde p(x_j)$. In some situations, the precise design of a randomized experiment may not be obvious. While most people would agree on how to randomize over medical treatment procedures, there is probably less agreement how to randomize over the tolerance of a person (does this include other changes of his personality, too?). Only sometimes, this problem can be resolved by including more variables and taking a less coarse-grained point of view. We do not go into further detail since we believe that this would require philosophical deliberations, which lie beyond the scope of this work. Instead, we may explicitly add the requirement that ``most people agree on what a randomized experiment should look like in this context''.

In general, there can be more than one true causal DAG. If one requires causal minimality, the true causal DAG is unique.
\begin{proposition} \label{prop:tcgunique}
Assume $\law{X_1,\ldots,X_p}$ has a density and
consider all true causal DAGs $\mathbb{G} := \{\G_{c,1}, \ldots, \G_{c,m}\}$ of $X_1, \ldots, X_p$. Then there is a partial order on $\mathbb{G}$ using the subgraph property $\leq$ as an ordering. This ordering has a least element $\Gc$, i.e., $\Gc \leq \G_{c,i}$ for all $i$. This element $\Gc$ is the unique true causal DAG such that $\lawX$ satisfies causal minimality with respect to $\Gc$.
\end{proposition}
\begin{proof}
See Appendix~\ref{app:proof_tcgunique}
\end{proof}

We now briefly comment on a true causal graph's behavior when some of the variables from the joint distribution are marginalized out.
\begin{ex}
\begin{itemize}
\item[(i)] If $X \leftarrow Z \rightarrow Y$ is the only true causal graph for $X, Y$ and $Z$, there is no true causal graph for the variables $X$ and $Y$ (the $do$-statements do not coincide).
\item[(ii)] Assume that the graph $X \rightarrow Y \rightarrow Z$ with additional $X \rightarrow Z$ is the only true causal graph for $X, Y$ and $Z$ and assume that $\law{X, Y, Z}$ is faithful with respect to this graph. Then, the only true causal graph for the variables $X$ and $Z$ is $X \rightarrow Z$.
\item[(iii)] If the situation is the same as in (ii) with the difference that $X \independent Z$ (i.e., $\law{X, Y, Z}$ is not faithful with respect to the true causal graph), the empty graph is also a true causal graph for $X$ and $Z$.
\end{itemize}
\end{ex}
Latent projections \citep{Verma1991} provide a formal way to obtain a true causal graph for marginalization. Cases (ii) and (iii) show that there are no purely graphical criteria that provide the {\it minimal} true causal graph described in Proposition~\ref{prop:tcgunique}. 

The results presented in the remainder of this paper can be understood without causal interpretation. Using these techniques to infer a true causal graph, however, requires the assumption 
%
that such a true causal DAG $\Gc$ for the observed distribution of $X_1, \ldots, X_p$ exists.
This includes the assumption that all ``relevant'' variables have been observed, sometimes called causal sufficiency, and that there are no feedback loops.

\citet{Richardson2002} introduce a representation of graphs (so-called Maximal Ancestral Graphs, or MAGs) with hidden variables that is closed under marginalization and conditioning. The FCI algorithm \citep{Spirtes2000} exploits the conditional independences in the data to partially reconstruct the graph. Less work concentrates on hidden variables in structural equation models \citep[e.g.,][]{Hoyer2008b, Janzing2009, SilvaGhahramani2009}.

\subsection{Structural Equation Models} \label{sec:sem}
A structural equation model (SEM) (also called a functional model) is defined as a tuple $(\SE, \lawN)$, where $\SE = (S_1, \ldots, S_p)$ is a collection of $p$ equations
\begin{equation} \label{eq:sem}
S_j: \quad X_j = f_j(\PA{j}, N_j)\,, \qquad j=1, \ldots, p
\end{equation}
and $\lawN = \law{N_1, \ldots, N_p}$ is the joint distribution of the noise variables, which we require to be jointly independent (thus, $\lawN$ is a product distribution) as we are assuming causal sufficiency. 
The $\PA{j}$ are considered the direct causes of $X_j$. An SEM specifies how the $\PA{j}$ affect $X_j$. Note that in physics (chemistry, biology, \dots), we would usually expect that such causal relationships occur in time, and are governed by sets of coupled differential equations. Under certain assumptions such as stable equilibria, one can derive an SEM describing how the equilibrium states of such a dynamical system will react to physical interventions on the observables involved (see \cite{MooijJS2013}). We do not deal with these issues in the present paper, but we take the SEM as our starting point. Moreover, we consider SEMs only for real-valued random variables $X_1, \ldots, X_p$.
The graph of a structural equation model is obtained simply by drawing direct edges from each parent to its direct effects, i.e., from each variable $X_k$ occurring on the right-hand side of equation \eqref{eq:sem} to $X_j$. We henceforth assume this graph to be acyclic. According to the notation defined in Section~\ref{sec:cg}, $\PA{j}$ are the parents of $X_j$. 
\citet{Pearl2009} shows in Theorem 1.4.1 that the law $\lawX$ generated by an SEM is Markov with respect to the graph.

Structural equation models contain strictly more information than their corresponding graph and law and hence also more information than the family of all interventional distributions together with the observational distribution. This information sometimes helps to answer counterfactual questions, as shown in the following example.
\begin{ex} \label{ex:cou}
Let $N_1, N_2 \sim Ber(0.5)$ and $N_3 \sim U(\{0,1,2\})$, such that the three variables are jointly independent. That is, $N_1, N_2$ have a Bernoulli distribution with parameter $0.5$ and $N_3$ is uniformly distributed on $\{0,1,2\}$. We define two different SEMs, first consider $\C{S}_A$:
$$
\C{S}_A = \left\{
\begin{array}{l}
X_1 = N_1\\
X_2 = N_2\\
X_3 = (1_{N_3 > 0} \cdot X_1 + 1_{N_3 = 0}\cdot X_2)\cdot 1_{X_1 \neq X_2} + N_3 \cdot 1_{X_1 = X_2}
\end{array}
\right.
$$ 
If $X_1$ and $X_2$ have different values, depending on $N_3$ we either choose $X_3 = X_1$ or $X_3 = X_2$. Otherwise $X_3 = N_3$. Now, $\C{S}_B$ differs from $\C{S}_A$ only in the latter case:
$$
\C{S}_B = \left\{
\begin{array}{l}
X_1 = N_1\\
X_2 = N_2\\
X_3 = (1_{N_3 > 0} \cdot X_1 + 1_{N_3 = 0}\cdot X_2)\cdot 1_{X_1 \neq X_2} + (2- N_3) \cdot 1_{X_1 = X_2}
\end{array}
\right.
$$ 
It can be checked that both SEMs generate the same observational distribution, which satisfies causal minimality with respect to the graph $X_1 \rightarrow X_3 \leftarrow X_2$. 
They also generate the same interventional distributions, for any possible intervention.
But the two models differ in a counterfactual statement\footnote{Here, we make use of Judea Pearl's definition of counterfactuals \citep{Pearl2009}.}. Suppose, we have seen a sample
$
(X_1, X_2, X_3) = (1,0,0) 
$
and we are interested in the counterfactual question, what $X_3$ would have been if $X_1$ had been $0$. From both $\C{S}_A$ and $\C{S}_B$ it follows that $N_3 = 0$, and thus the two SEMs ``predict'' different values for $X_3$ under a counterfactual change of $X_1$.
\end{ex}
If we want to use an estimated SEM to predict counterfactual questions, this example shows that we require assumptions that let us distinguish between $\C{S}_A$ or $\C{S}_B$. In this work we exploit the additive noise assumption to infer the \emph{structure} of an SEM. We do not claim that we can predict counterfactual statements.

Another property of structural equation models is that they have the power to describe many distributions\footnote{A similar but weaker statement than Proposition~\ref{prop:sur} can be found in \citep{Druzdzel2001, Janzing2010a}.}. 
\begin{proposition} \label{prop:sur}
Consider $X_1, \ldots, X_p$ and let $\lawX$ be Markov with respect to $\G$. Then there exists an SEM $(\SE, \lawN)$ with graph $\G$ that generates the distribution $\lawX$.
\end{proposition}
\begin{proof}
See Appendix~\ref{app:sems}.
\end{proof}

Structural equation models have been used for a long time in fields like agriculture or social sciences \citep[e.g.,][]{Wright1921, Bollen1989}. Model selection, for example, was done by fitting different structures that were considered as reasonable given the prior knowledge about the system. These candidate structures were then compared using goodness of fit tests. 
In this work we instead consider the question of identifiability, which has not been addressed until recently.
\begin{problem}[population case]
Suppose we are given a distribution $\lawX = \law{X_1, \ldots, X_p}$ that has been generated by an (unknown) structural equation model with graph $\G_0$; in particular, $\lawX$ is Markov with respect to $\G_0$. Can the (observational) distribution $\lawX$ be generated by a structural equation model with a different graph $\G \neq \G_0$? If not, we call $\G_0$ identifiable from $\lawX$.
\end{problem}
In general, $\G_0$ is not identifiable from $\lawX$: the joint distribution $\lawX$ is certainly Markov with respect to a lot of different graphs, e.g., to all fully connected acyclic graphs. Proposition~\ref{prop:sur} states the existence of corresponding SEMs. 
What can be done to overcome this indeterminacy? The hope is that by using additional assumptions one obtains restricted models, in which we can identify the graph from the joint distribution. 
Considering graphical models, we see in Section~\ref{sec:ibm} how the assumption that $\lawX$ is Markov and faithful with respect to $\G_0$ leads to identifiability of the Markov equivalence class of $\G_0$.
Considering SEMs, we see in Section~\ref{sec:ide} that additive noise models as a special case of restricted SEMs even lead to identifiability of the correct DAG. Also Section~\ref{sec:lin} contains such a restriction based on SEMs.

\section{Alternative Methods} \label{sec:exi}

\subsection{Estimating the Markov Equivalence Class: Independence-Based Methods} \label{sec:ibm}
Conditional independence-based methods
like the PC algorithm and the FCI algorithm \citep{Spirtes2000} assume that $\lawX$ is Markov and faithful with respect to the correct graph $\G_0$ (that means \emph{all} conditional independences in the joint distribution are entailed by the Markov condition, cf. Section~\ref{sec:cg}). 
Since both assumptions put restrictions only on the conditional independences in the joint distribution, these methods are not able to distinguish between two graphs that entail exactly the same set of (conditional) independences, i.e., between Markov equivalent graphs.
Since many Markov equivalence classes contain more than one graph, conditional independence-based methods thus usually leave some arrows undirected and cannot uniquely identify the correct graph.


The first step of the PC algorithm determines the variables that are adjacent. One therefore has to test whether two variables are dependent given \emph{any} other subset of variables. The PC algorithm exploits a very clever procedure to reduce the size of the condition set. In the worst case, however, one has to perform conditional independence tests with conditioning sets of up to $p-2$ variables (where $p$ is the number of variables in the graph). 
Although there is recent work on kernel-based conditional independence tests \citep{Fukumizu2008, Zhang2011}, such tests are difficult to perform in practice if one does not restrict the variables to follow a Gaussian distribution, for example \citep[e.g.,][]{Bergsma2004}.

To prove consistency of the PC algorithm one does not only require faithfulness, but strong faithfulness \citep{Zhang2003, Kalisch2007}. \citet{Uhler2013} argue that this is a restrictive condition. Since parts of faithfulness can be tested given the data \citep{Zhang2008}, the condition may be weakened.

From our perspective independence-based methods face the following challenges: (1) We can identify the correct DAG only up to Markov equivalence classes. (2) Conditional independence testing, especially with a large conditioning set, is difficult in practice. (3) Simulation experiments suggest, that in many cases, the distribution is close to unfaithfulness. In these cases there is no guarantee that the inferred graph(s) will be close to the original one.

\subsection{Estimating the Markov Equivalence Class: Score-Based Methods} \label{sec:scorebm}
Although the roots for score-based methods for causal inference may date back even further, we mainly refer to \citet{Geiger1994}, \citet{Heckerman1997} and \citet{Chickering2002} and references therein. Given the data $\mathcal{D}$ from a vector $\X$ of variables, i.e., $n$ i.i.d.\ samples, the idea is to assign a score $S(\mathcal{D}, \G)$ to each graph $\G$ and search over the space of DAGs for the best scoring graph.
\begin{equation} \label{eq:sbmopt}
\hat{\G} := \argmax_{\G \text{ DAG over } \X} S(\mathcal{D}, \G)
\end{equation}
There are several possibilities to define such a scoring function. Often a parametric model is assumed (e.g., linear Gaussian equations or multinomial distributions), which introduces a set of parameters $\theta \in \mathbf{\Theta}$.\\

From a Bayesian point of view, we may define priors $p_{pr}(\G)$ and $p_{pr}(\theta)$ over DAGs and parameters and consider the log posterior as a score function, or equivalently (note that $p(\mathcal{D})$ is constant over all DAGs):
\begin{equation*}
S(\mathcal{D}, \G) := \log p_{pr}(\G) + \log p(\mathcal{D}|\G)\,,
\end{equation*}
where $p(\mathcal{D}|\G)$ is the marginal likelihood 
\begin{equation*}
p(\mathcal{D}|\G) = \int_{\mathbf{\Theta}} p(\mathcal{D}|\G, \theta) \cdot p_{pr}(\theta)\,d\theta.
\end{equation*}
In this case, $\hat{\G}$ defined in~\eqref{eq:sbmopt} is the mode of the posterior distribution, which is usually called the \emph{maximum a posteriori} (or MAP) estimator. Instead of a MAP estimator, one may be interested in the full posterior distribution over DAGs. This distribution can subsequently be averaged over all graphs to get a posterior of the hypothesis about the existence of a specific edge, for example.

In the case of parametric models, we call two graphs $\G_1$ and $\G_2$ {\it distribution equivalent} if for each parameter $\theta_1 \in \mathbf{\Theta}_1$ there is a corresponding parameter $\theta_2 \in \mathbf{\Theta_2}$, such that the distribution obtained from $\G_1$ in combination with $\theta_1$ is the same as the distribution obtained from graph $\G_2$ with $\theta_2$, and vice versa. It is known that in the linear Gaussian case (or for unconstrained multinomial distributions) two graphs are distribution-equivalent if and only if they are Markov equivalent. One may therefore argue that $p(\mathcal{D}|\G_1)$ and $p(\mathcal{D}|\G_2)$ should be the same for Markov equivalent graphs $\G_1$ and $\G_2$. \citet{Heckerman1995} discuss how to choose the prior over parameters accordingly.\\

Instead, we may consider the maximum likelihood estimator $\hat{\theta}$ in each graph and define a score function by using a penalty, e.g., the Bayesian Information Criterion (BIC):
\begin{equation*}
S(\mathcal{D}, \G) = \log p(\mathcal{D}|\hat{\theta},\G) - \frac{d}{2} \log n\,,
\end{equation*}
where $n$ is the sample size and $d$ the dimensionality of the parameter $\theta$.

Since the search space of all DAGs is growing super-exponentially in the number of variables \citep[e.g.,][]{Chickering2002}, greedy search algorithms are applied to solve equation~\eqref{eq:sbmopt}: at each step there is a candidate graph and a set of neighboring graphs. For all these neighbors one computes the score and considers the best-scoring graph as the new candidate. If none of the neighbors obtains a better score, the search procedure terminates (not knowing whether one obtained only a local optimum). Clearly, one therefore has to define a neighborhood relation. Starting from a graph $\G$, we may define all graphs as neighbors from $\G$ that can be obtained by removing, adding or reversing one edge.
In the linear Gaussian case, for example, one cannot distinguish between Markov equivalent graphs. It turns out that in those cases it is beneficial to change the search space to Markov equivalence classes instead of DAGs. 
The greedy equivalence search (GES) \citep{Meek1997,Chickering2002} 
starts with the empty graph and consists of two-phases. In the first phase, edges are added until a local maximum is reached; in the second phase, edges are removed until a local maximum is reached, which is then given as an output of the algorithm. \citet{Chickering2002} proves consistency of this method by using consistency of the BIC \citep{Haughton1988}. 
 
\subsection{Estimating the DAG: LiNGAM} \label{sec:lin}
\citet{Kano2003} and \citet{Shimizu2006} propose an inspiring method exploiting non-Gaussianity of the data\footnote{A more detailed tutorial can be found on \url{http://www.ar.sanken.osaka-u.ac.jp/~sshimizu/papers/Shimizu13BHMK.pdf}.}.
Although their work covers the general case, the idea is maybe best understood in the case of two variables:
\begin{ex}
Suppose
\begin{equation*}
Y= \phi X + \eps, \quad \eps \independent X\,,
\end{equation*}
where $X$ and $\eps$ are normally distributed. It is easy to check that  
\begin{equation*}
X= \tilde \phi Y + \tilde \eps, \quad \tilde \eps \independent Y\,.
\end{equation*}
with $\tilde \phi = \frac{\phi \var(X)}{\phi ^2 \var(X) + \sigma^2} \neq \frac{1}{\phi}$ and $\tilde \eps = X - \tilde \phi Y$.
\end{ex}
If we consider non-Gaussian noise, however, the structural equation model becomes identifiable.
\begin{proposition} \label{prop:linnng}
Let $X$ and $Y$ be two random variables, for which
\begin{align*}
Y&=\phi X + \eps, \quad \eps \independent X, \; \phi \neq 0
\end{align*}
holds.
Then we can reverse the process, i.e., there exists $\psi \in \R$ and a noise $\tilde \eps$, such that
$$X= \psi Y + \tilde \eps, \quad \tilde \eps \independent Y\,,$$
if and {\emph only if} $X$ and $\eps$ are Gaussian distributed.
\end{proposition}
\citet{Shimizu2006} were the first to report this result. They prove it even for more than two variables using Independent Component Analysis (ICA) \citep[][Theorem~11]{ica_paper}, which itself is proved using the Darmois-Skitovi\v{c} theorem \citep{skitovich, skitovich_trans, darmois}. Alternatively, Proposition~\ref{prop:linnng} can be proved directly using the Darmois-Skitovi\v{c} theorem \citep[e.g.,][Theorem~2.10]{PetersDiploma}. 
\begin{theorem}[\citet{Shimizu2006}]
Assume a linear SEM with graph $\G_0$
\begin{equation} \label{eq:lingam}
X_j = \sum_{k \in \PA[\G_0]{j}} \beta_{jk} X_k + N_j\,,\qquad j=1, \ldots, p
\end{equation}
where all $N_j$ are jointly independent and non-Gaussian distributed.
Additionally, for each $j \in \{1, \ldots, p\}$ we require $\beta_{jk} \neq 0$ for all $k \in \PA[\G_0]{j}$.
Then, the graph $\G_0$ is identifiable from the joint distribution.
\end{theorem}
The authors call this model a linear non-Gaussian acyclic model (LiNGAM) and provide a practical method based on ICA that can be applied to a finite amount of data. 
Later, improved versions of this method have been proposed in \citep{Shimizu2011, Hyvarinen2013}.

\subsection{Estimating the DAG: Gaussian SEMs with Equal Error Variances} \label{sec:sev} 
There is another deviation from linear Gaussian SEMs that makes the graph identifiable. 
\citet{Peters2014biom} show that restricting the error (or noise) variables to have the same variance is sufficient to recover the graph structure.
\begin{theorem}[\citet{Peters2014biom}] \label{thm:sev}
Assume an SEM with graph $\G_0$
\begin{equation} \label{eq:mmm}
X_j = \sum_{k \in \PA[\G_0]{j}} \beta_{jk} X_k + N_j\,,\qquad j=1, \ldots, p
\end{equation}
where all $N_j$ are i.i.d.\ and follow a Gaussian distribution.
Additionally, for each $j \in \{1, \ldots, p\}$ we require $\beta_{jk} \neq 0$ for all $k \in \PA[\G_0]{j}$.
Then, the graph $\G_0$ is identifiable from the joint distribution.
\end{theorem}
For estimating the coefficients 
$\beta_{jk}$ and the error variance $\sigma^2$, 
\citet{Peters2014biom} propose to use a penalized maximum likelihood method (BIC). 
For optimization they propose a greedy search algorithm in the space of DAGs.  
Rescaling the variables changes the error terms. Therefore, in many applications Theorem~\ref{thm:sev} cannot be sensibly applied. The BIC criterion, however, always allows to compare the method's score with the score of a linear Gaussian SEM that uses more parameters and does not make the assumption of equal error variances.

\section{Identifiability of Continuous Additive Noise Models} \label{sec:ide}
Recall that equation~\eqref{eq:sem} defines the general form of an 
SEM:
$X_j = f_j(\PA{j},N_j)\,, j=1, \ldots, p$
with jointly independent variables $N_i$.
We have seen that these models are too general to identify the graph (Proposition~\ref{prop:sur}). It turns out, however, that
constraining the function class leads to identifiability. As a first step we restrict the form of the function to be additive 
with respect to the noise variable:
\begin{equation} \label{eq:anm}
X_j = f_j(\PA{j}) + N_j\,, \qquad j=1, \ldots, p
\end{equation}
and assume that all noise variables $N_j$ have a strictly positive density.
For those models with strictly positive density, causal minimality reduces to the condition that each function $f_j$ is not constant in any of its arguments.
\begin{proposition} \label{prop:nonconst}
Consider a distribution generated by a model~\eqref{eq:anm} and assume that the functions $f_j$ are not constant in any of its arguments, i.e., for all $j$ and $i \in \PA[]{j}$ there are some $x_{\PA[]{j} \setminus \{i\}}$ and some $x_i \neq x'_i$ such that 
$$
f_j(x_{\PA[]{j} \setminus \{i\}}, x_i) \neq 
f_j(x_{\PA[]{j} \setminus \{i\}}, x'_i)\,.
$$
Then the joint distribution satisfies causal minimality with respect to the corresponding graph. Conversely, if there is a $j$ and $i$ such that $f_j(x_{\PA[]{j} \setminus \{i\}}, \cdot)$ is constant, causal minimality is violated.
\end{proposition}
\begin{proof}
See Appendix~\ref{app:proofnonconst}
\end{proof}
Linear functions and Gaussian variables identify only the correct Markov equivalence class and not necessarily the correct graph. In the remainder of this section we establish results showing that this is an exceptional case. We develop conditions that guarantee the identifiability of the DAG. Proposition~\ref{prop:biv1} indicates that this condition is rather weak.

Throughout this section we assume that all random variables are absolutely continuous with respect to the Lebesgue measure. \citet{Peters2011a} provides an extension for variables that are absolutely continuous with respect to the counting measure.

\subsection{Bivariate Additive Noise Models}
We now add another assumption about the form of the structural equations. 
\begin{definition} \label{def:ibanm}
Consider an additive noise model \eqref{eq:anm} with two variables, i.e., the two equations
$X_i = N_i$ and $X_j = f_j(X_i) + N_j$ with $\{i,j\} = \{1,2\}$. 
We call this SEM an \emph{identifiable bivariate additive noise model}
if the triple $(f_j, \law{X_i}, \law{N_j})$ 
satisfies Condition~\ref{cond}. In particular, we require the noise variables to have strictly positive densities.
\end{definition}

\begin{condition} \label{cond}
The triple $(f_j, \law{X_i}, \law{N_j})$  does not solve the following differential equation
for all $x_i,x_j$ with $\nu''(x_j-f(x_i))f'(x_i)\neq 0$:
\begin{align}\label{DGL}
\xi'''&=   \xi''  \left(-\frac{\nu'''f'}{\nu''}
+\frac{f''}{f'}\right) 
-2 \nu''f''f' 
+\nu'f'''+\frac{\nu'\nu'''f''f'}{\nu''}-\frac{\nu'(f'')^2}{f'}\,,
\end{align}
Here, $f := f_j$, and $\xi:=\log p_{X_i}$ and $\nu:=\log p_{N_j}$ are the logarithms of the strictly positive densities.
To improve readability, we have skipped the arguments $x_j-f(x_i)$, $x_i$, and $x_i$ 
for $\nu$, $\xi,$ and $f$ and their derivatives, respectively.
\end{condition}
\citet{Zhang2009} even allow for a bijective transformation of the data, i.e., $X_j = g_j(f_j(X_i) + N_j)$ and obtain a similar differential equation as~\eqref{DGL}.

As the name in Definition~\ref{def:ibanm} already suggests, we have identifiability for this class of SEMs.
\begin{theorem}\label{thm:biv}
Let $\lawX = \law{X_1, X_2}$ be generated by an identifiable bivariate additive noise model with graph $\G_0$ and assume causal minimality, i.e., a non-constant function $f_j$ (Proposition~\ref{prop:nonconst}). Then, $\G_0$ is identifiable from the joint distribution.
\end{theorem}
\begin{proof}
The proof of \citet{Hoyer2008} is reproduced in Appendix~\ref{app:proof_biv}.
\end{proof}
Intuitively speaking, we expect a ``generic'' triple $(f_j, \law{X_i}, \law{N_j})$ to satisfy Condition~\ref{cond}. The following proposition presents one possible formalization. After fixing $(f_j,\law{N_j})$ we consider the space of all distributions $p_X$ such that Condition~\ref{cond} is violated. This space is contained in a three dimensional space. Since the space of continuous distributions is infinite dimensional, we can therefore say that Condition~\ref{cond} is satisfied for ``most distributions'' $\law{X_i}$.
\begin{proposition} \label{prop:biv1}
If for a fixed pair $(f_j,\law{N_j})$ 
there exists $y\in \R$ such that $\nu''(y-f(x))f'(x)\neq 0$ for 
all but a  countable set of points
$x\in \R$,
the set of all $p_X$ for which $(f_j, \law{X_i}, \law{N_j})$ does not satisfy Condition~\ref{cond}
is contained in a $3$-dimensional space. 
\end{proposition}
The condition $\nu''(y-f(x))f'(x) \neq 0$ holds for all $x$ if there is no interval where $f$ is constant and the logarithm of the noise density is not linear, for example.
\begin{proof}
See Appendix~\ref{app:proof_biv1}.
\end{proof}
In the case of Gaussian variables, the differential equation \eqref{DGL} simplifies. We thus have the following result.   
\begin{corollary}\label{cor:gaussianity_implies_linearity}
If $X_i$ and $N_j$ follow a Gaussian distribution and $(f_j, \law{X_i}, \law{N_j})$ does not satisfy Condition~\ref{cond}, then $f_j$ is linear.
\end{corollary}
\begin{proof}
See Appendix~\ref{app:proof_gaussianity_implies_linearity}.
\end{proof}
Although non-identifiable cases are rare, the question remains when identifiability is violated. \citet{Zhang2009} prove that non-identifiable additive noise models necessarily fall into one out of five classes. 
\begin{proposition}[\citet{Zhang2009}] \label{prop:kun}
Consider $X_2=f_2(X_1)+N_2$ with fully supported noise variable $N_2$ that is independent of $X_1$ and three times differentiable function $f_2$. Let further $\frac{d}{dx_1}f_2(x_1)\frac{d^2}{d x_1^2} \log p_{N_2}(x_2) = 0$ only at finitely many points $(x_1, x_2)$. If there is a backward model, i.e., we can write $X_1=g_1(X_2)+\tilde N_1$ with $\tilde N_1$ independent of $X_2$, then one of the following must hold.
\begin{itemize}
\item[I.] $X_1$ is Gaussian, $N_2$ is Gaussian and $f$ is linear.
\item[II.] $X_1$ is log-mix-lin-exp, $N_2$ is log-mix-lin-exp and $f$ is linear.
\item[III.] $X_1$ is log-mix-lin-exp, $N_2$ is one-sided asymptotically exponential and $f$ is strictly monotonic with $f'(x_1) \rightarrow 0$ as $x_1 \rightarrow \infty$ or as $x_1 \rightarrow -\infty$.
\item[IV.] $X_1$ is log-mix-lin-exp, $N_2$ is generalized mixture of two exponentials and $f$ is strictly monotonic with $f'(x_1) \rightarrow 0$ as $x_1 \rightarrow \infty$ or as $x_1 \rightarrow -\infty$.
\item[V.] $X_1$ is generalized mixture of two exponentials, $N_2$ is two-sided asymptotically exponential and $f$ is strictly monotonic with $f'(x_1) \rightarrow 0$ as $x_1 \rightarrow \infty$ or as $x_1 \rightarrow-\infty$. 
\end{itemize}
\end{proposition}
Precise definitions can be found in Appendix~\ref{app:kun}. In particular, we obtain identifiability whenever the function $f$ is not injective. 
Proposition~\ref{prop:kun} states that belonging to one of these classes is a necessary condition for non-identifiability. We now show sufficiency for two classes. The linear Gaussian case is well-known and easy to prove.
\begin{ex}\label{ex:cou1}
Let $X_2 = a X_1 + N_2$ with independent $N_2\sim \mathcal{N}(0,\sigma^2)$ and $X_1 \sim \mathcal{N}(0,\tau^2)$. 
We can then consider all variables in $\mathcal{L}_2$ and project $X_1$ onto $X_2$. This leads to an orthogonal decomposition $X_1 = \tilde a X_2 + \tilde N_1$. Since for jointly Gaussian variables uncorrelatedness implies independence, we obtain a backward additive noise model. Figure~\ref{fig:counter} (left) shows the joint density and the functions for the forward and backward model.
\end{ex}
We also give an example of a nonidentifiable additive noise model with non-Gaussian distributions, where the forward model is described by case II, and the backwards model by case IV:
\begin{ex} \label{ex:cou2}
Let $X_2 = a X_1 + b + N_2$ with independent log-mix-lin-exp $N_2$ and $X_1$, i.e., we have the log-densities
$$
\xi(x) = \log p_{X_2}(x) = c_1 \exp(c_2 x) + c_3 x + c_4
$$
and
$$
\nu(x) = \log p_{N_2}(n) = \tc_1 \exp(\tc_2 n) + \tc_3 n + \tc_4 \,.
$$
Then $X_2$ is a generalized mixture of exponential distributions. If and only if $c_2 = -a \tc_2$ and $c_3 \neq a \tc_3$ we obtain a valid backward model $X_1 = g(X_2) + \tilde N_1$ with log-mix-lin-exp $\tilde N_1$. Again, Figure~\ref{fig:counter} (right) shows the joint distribution over $X_1$ and $X_2$ and forward and backward functions.
\end{ex}
\begin{proof}
See Appendix~\ref{app:cou2}.
\end{proof}
\begin{figure}
\includegraphics[width=0.46\textwidth]{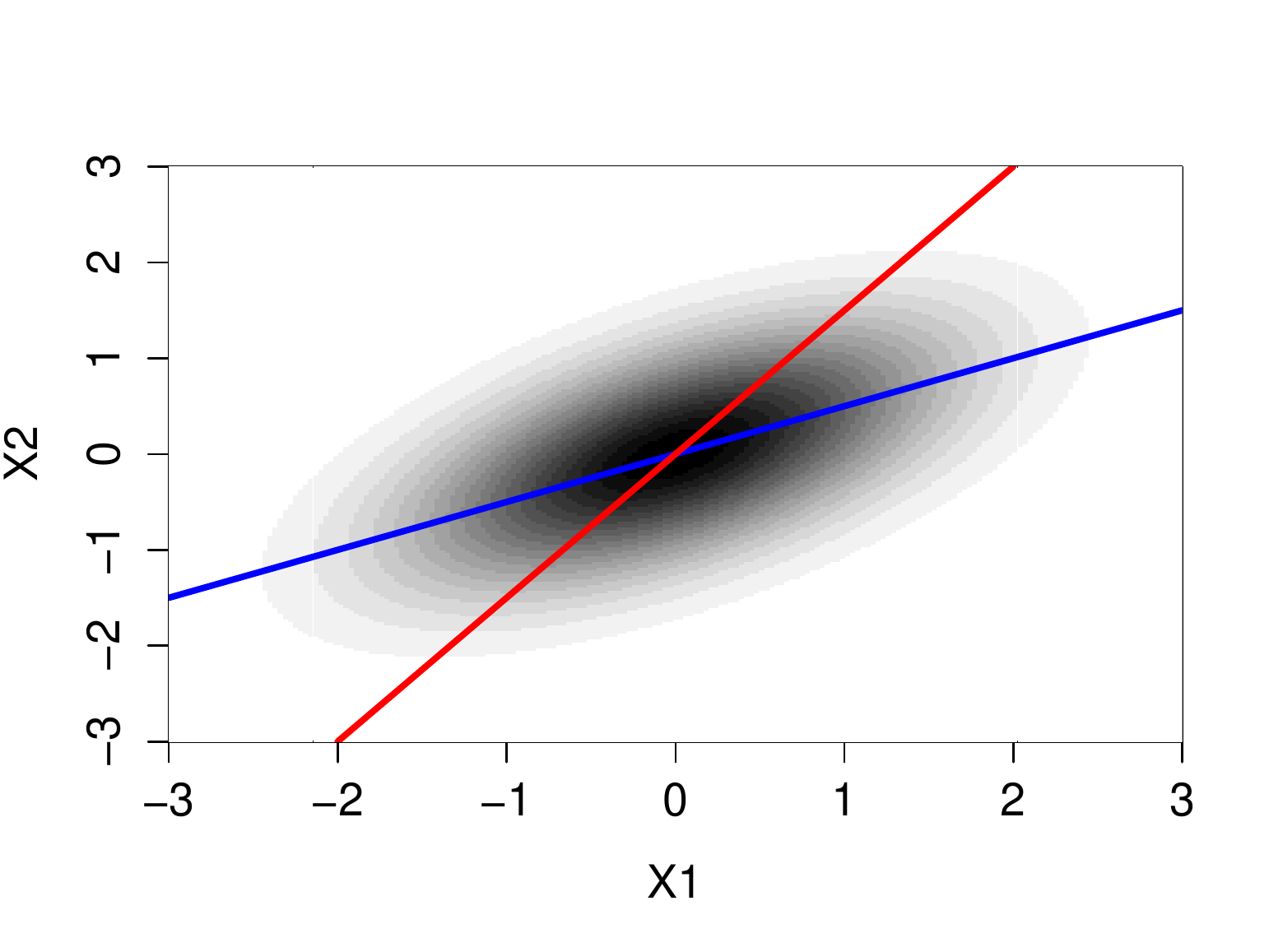}
\hfill
\includegraphics[width=0.46\textwidth]{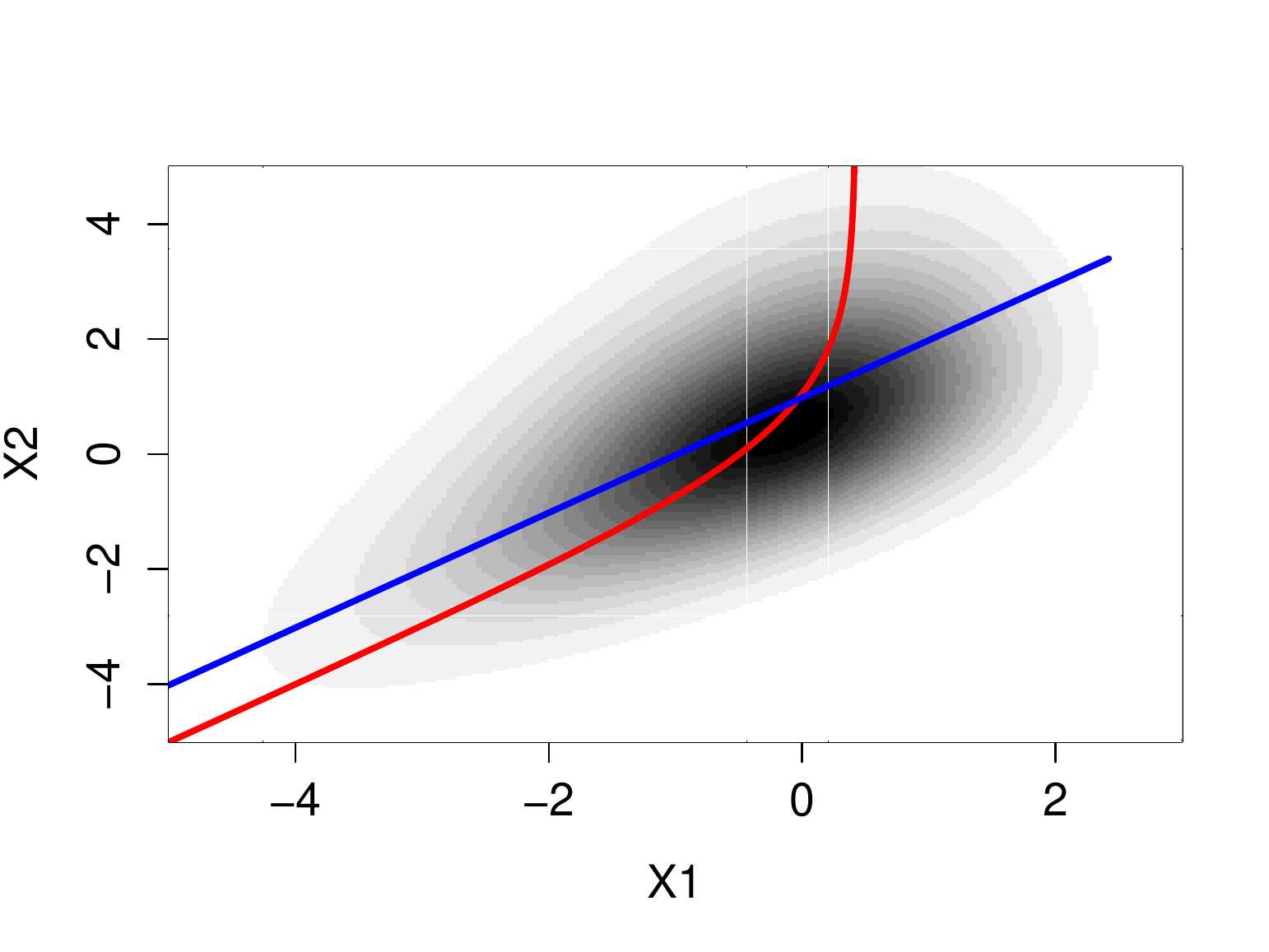}
\caption{Joint density over $X_1$ and $X_2$ for two non-identifiable examples. The left panel shows Example~\ref{ex:cou1} (linear Gaussian case) and the right panel shows Example~\ref{ex:cou2} (the latter plot is based on kernel density estimation). The blue function corresponds to the forward model $X_2 = f(X_1) + N_2$, the red function to the backward model $X_1 = g(X_2)+\tilde N_1$.}
\label{fig:counter}
\end{figure}
Example~\ref{ex:cou2} shows how parameters of function, input and noise distribution have to be ``fine-tuned'' to yield non-identifiability \citep{Steudel2010}. 

It can be shown that bivariate identifiability even holds generically when causal feedback is allowed (i.e., if both $X$ causes $Y$ \emph{and} $Y$ causes $X$), at least when assuming noise and input distributions to be Gaussian \citep{Mooij_et_al_NIPS_11}.

\subsection{From Bivariate to Multivariate Models}
It turns out that Condition~\ref{cond} also suffices to prove identifiability in the multivariate case. Assume we are given $p$ structural equations $X_j = f_j(\PA{j}) + N_j$ as in~\eqref{eq:anm}. If we fix all arguments of the functions $f_j$ except for one parent and the noise variable, we obtain a bivariate model.
One may expect that it suffices to put restrictions like Condition~\ref{cond} on this triple of function, input and noise distribution. 
This is not the case. 
\begin{ex} \label{ex:motivating}
Consider the following SEM
$$
X_1 = N_1,\quad
X_2 = f_2(X_1) + N_2,\quad
X_3 = f_3(X_1) + a \cdot X_2 + N_3
$$
with 
$N_1 \sim \mathcal{U}[0,1]$, 
$N_2 \sim \mathcal{N}(0,\sigma_2^2)$ and
$N_3 \sim \mathcal{N}(0,\sigma_3^2)$, i.e., $N_1$ is uniformly distributed on $[0,1]$ and $N_2$ and $N_3$ are normally distributed. The variables $X_2$ and $X_3$ themselves are non-Gaussian but
$$
X_3 \given_{X_1=x_1} = c+ a \cdot X_2 \given_{X_1=x_1} + N_3
$$ 
is a linear Gaussian equation for all $x_1$. We can revert this equation and obtain the same joint distribution by an SEM of the form
$$
X_1 = M_1,\quad
X_2 = g_2(X_1) + b \cdot X_3 + M_2,\quad
X_3 = g_3(X_1) + M_3
$$
for some 
$M_1 \sim \mathcal{U}[0,1]$,
$M_2 \sim \mathcal{N}(0,\tilde \sigma_2^2)$ and 
$M_3 \sim \mathcal{N}(0,\tilde \sigma_3^2)$.
Thus, the DAG is not identifiable from the joint distribution.
\end{ex}
Instead, we need to put restrictions on conditional distributions.
\begin{definition} \label{def:wnn}
Consider an additive noise model \eqref{eq:anm} with $p$ variables. We call this SEM a \emph{restricted additive noise model} 
if for all $j \in \B{V}$, $i \in \PA{j}$ and all sets 
$\B{S} \subseteq \B{V}$ with 
$\PA{j} \setminus \{i\} \subseteq \B{S} \subseteq \ND{j} \setminus \{i,j\}$, there is an $x_{\B{S}}$ with $p_{\B{S}}(x_{\B{S}}) > 0$, s.t.
\begin{equation*}
\Big(f_j(x_{\PA{j}\setminus \{i\}}, \underbrace{\cdot}_{X_i}), \law{X_i \given X_{\B{S}}=x_{\B{S}}}, \law{N_j}\Big)
\end{equation*}
satisfies Condition~\ref{cond}. Here, the underbrace indicates the input component of $f_j$ for variable $X_i$.
In particular, we require the noise variables to have non-vanishing densities and the functions $f_j$ to be continuous and three times continuously differentiable.
\end{definition}
Assuming causal minimality, we can identify the structure of the SEM from the distribution.
\begin{theorem}\label{thm:mul}
  Let $\lawX = \law{X_1, \ldots, X_p}$ be generated by a restricted additive noise model with graph $\G_0$ and assume that $\lawX$ satisfies causal minimality with respect to $\G_0$, i.e., the functions $f_j$ are not constant 
(Proposition~\ref{prop:nonconst}). Then, $\G_0$ is identifiable from the joint distribution.
\end{theorem}
\begin{proof}
See Appendix \ref{app:proof_mul}.
\end{proof}
Our proof of Theorem~\ref{thm:mul} contains a graphical statement that turns out to be a main argument for proving identifiability for Gaussian models with equal error variances \citep{Peters2014biom}. We thus state it explicitly as a proposition.
\begin{proposition}\label{prop:graph}
Let $\G$ and $\Gp$ be two different DAGs over variables $\X$.
\begin{enumerate}
\item[(i)] Assume that $\lawX$ has a strictly positive density and satisfies the Markov condition and causal minimality with respect to $\G$ and $\Gp$. Then there are variables $L, Y \in \B{X}$ such that for the sets
$\B{Q}:=\PA[\G]{L} \setminus \{Y\}$, $\B{R}:= \PA[\Gp]{Y} \setminus \{L\}$ and $\B{S}:=\B{Q} \cup \B{R}$
we have
\begin{itemize}
\item $Y \rightarrow L$ in $\G$ and $L \rightarrow Y$ in $\Gp$
\item $\B{S} \subseteq \ND[\G]{L} \setminus \{Y\}$ and $\B{S} \subseteq \ND[\Gp]{Y} \setminus \{L\}$ 
\end{itemize}
\item[(ii)] In particular, if $\lawX$ is Markov and faithful with respect to $\G$ and $\Gp$ (i.e., both graphs belong to the same Markov equivalence class), there are variables $L, Y$ such that
\begin{itemize}
\item $Y \rightarrow L$ in $\G$ and $L \rightarrow Y$ in $\Gp$
\item $\PA[\G]{L} \setminus \{Y\}  = \PA[\Gp]{Y} \setminus \{L\}$
\end{itemize}
\end{enumerate}
\end{proposition}
\begin{proof}
See Appendix \ref{app:propmain}.
\end{proof}
If the distribution is Markov and faithful with respect to the underlying graph it is known that we can recover the correct Markov equivalence class. \citet{Chickering1995} proves that two graphs within this Markov equivalence class can be transformed into each other by a sequence of so-called covered edge reversals. This result implies part (ii) of the proposition. Part (i) establishes a similar statement when replacing faithfulness by causal minimality.

Although Theorem~\ref{thm:mul} is stated for additive noise models, it can be seen as an example of a more general principle.
\begin{remark}
Theorem~\ref{thm:mul} is not limited to restricted additive noise models. Whenever we have a restriction like Condition~\ref{cond} that ensures identifiability in the bivariate case (Theorem~\ref{thm:biv}), the multivariate version (Theorem~\ref{thm:mul}) remains valid. The proof we provide in the appendix stays exactly the same. The algorithms in Section~\ref{sec:alg}, however, use standard regression methods and therefore rely on the additive noise assumption.
\end{remark}
The result can therefore also be used to prove identifiability of SEMs that are restricted to discrete additive noise models \citep{Peters2011a} or post-nonlinear additive noise models \citep{Zhang2009}. In the latter model class we allow a bijective nonlinear distortion: $X_j = g_j\big(f_j(\PA{j}) + N_j \big)$. These models allow for more complicated functional relationships but are harder to fit from empirical data than the additive noise models considered in this work.

We explicitly state one specific identifiability result 
that we believe to constitute an important model class for applications.
Without giving an identifiability result like Corollary~\ref{cor:new} \citet{Tamada2011} have already used this result for structure learning \citep[see also][]{Tamada2011b}.
Lemma~6 of \citet{Zhang2009} implies that Theorem~\ref{thm:mul} remains valid if we replace Condition~\ref{cond} in Definition~\ref{def:wnn} by the condition that $f_j$ is nonlinear and $\mathcal{L}(N_j)$ is Gaussian. We formulate this as a corollary.
\begin{corollary} \label{cor:new}
\begin{enumerate}
\item[(i)]
Let $\lawX = \law{X_1, \ldots, X_p}$ be generated by an SEM with
\begin{equation} 
X_j = f_{j}(X_{\PA{j}}) + N_j 
\end{equation}
with normally distributed noise variables $N_j \sim \mathcal{N}(0,\sigma_j^2)$ and three times differentiable functions $f_{j}$ that are not linear in any component: denote the parents of $X_j$ by $X_{k_1}, \ldots, X_{k_{\ell}}$, then the function $f_{j}(x_{k_1}, \ldots, x_{k_{a-1}}, \cdot, x_{k_{a+1}}, \ldots, x_{k_{\ell}})$ is assumed to be nonlinear for all $a$ and some $x_{k_1}, \ldots, x_{k_{a-1}},$ $x_{k_{a+1}}, \ldots, x_{k_{\ell}} \in \R^{\ell - 1}$.
\item[(ii)]
As a special case, 
let $\lawX = \law{X_1, \ldots, X_p}$ be generated by an SEM with
\begin{equation} \label{eq:semaddadd}
X_j = \sum_{k \in \PA{j}} f_{j,k}(X_{k}) + N_j 
\end{equation}
with normally distributed noise variables $N_j \sim \mathcal{N}(0,\sigma_j^2)$ and three times differentiable, nonlinear functions $f_{j,k}$. 
\end{enumerate}
\noindent
In both cases (i) and (ii), we can identify the corresponding graph $\G_0$ from the distribution $\lawX$. 

Both statements remain true if the noise distributions for source nodes, i.e., nodes with no parents, are allowed to have a non-Gaussian density with full support on the real line $\R$.
\end{corollary}
\begin{proof}
See Appendix \ref{app:proofcornew}.
\end{proof}

Theorem~\ref{thm:mul} requires the positivity of densities in order to make use of the intersection property of conditional independence. \citet{Petersinters} shows that the intersection property still holds under weaker assumptions. It also discusses fundamental limits of causal inference when positivity is violated.

\subsection{Estimating the Topological Order} \label{sec:toporder}
We now investigate the case when we drop the assumption of causal minimality. Assume therefore that we are given a distribution $\lawX$ from an additive noise model with graph $\G_0$. We cannot recover the correct graph $\G_0$ because we can always add edges $i \rightarrow j$ or remove edges that ``do not have any effect'' without changing the distribution. This is formalized by the following lemma.
\begin{lemma} \label{lem:nocm}
Let $\lawX$ be generated by an additive noise model with graph $\G_0$. 
\begin{enumerate}
\item[(a)] For each supergraph $\G \geq \G_0$ there is an additive noise model that leads to the distribution $\lawX$.
\item[(b)] For each subgraph $\G \leq \G_0$ such that $\lawX$ is Markov with respect to $\G$ there is an additive noise model that leads to the distribution $\lawX$. Furthermore, there is an additive noise model with unique graph $\G_0^{\text{min}} \leq \G_0$ that leads to $\lawX$ and satisfies causal minimality.
\end{enumerate}
\end{lemma}

\begin{proof}
See Appendix~\ref{app:prooftoporder}.
\end{proof}
Despite this indeterminacy we can still recover the correct order of the variables. Given a permutation $\pi \in S_p$ on $\{1, \ldots, p\}$ we therefore define the fully connected DAG $\G^{\text{full}}_{\pi}$ by the DAG that contains all edges $\pi(i) \rightarrow \pi(j)$ for $i<j$. 
As a direct consequence of Theorem~\ref{thm:mul} and Lemma~\ref{lem:nocm} we can identify the set of true orderings:
\begin{corollary} \label{cor:toporder}
Let $\lawX = \law{X_1, \ldots, X_p}$ be generated by an additive noise model with graph $\G_0$. Assume that the SEM corresponding to the minimal graph $\G^{\text{min}}_0$ defined as in Lemma~\ref{lem:nocm} (b) is a restricted additive noise model. We can then identify the set $\Pi^0$ of true orderings
$$
\Pi^0 := \{\pi \in S_p\ |\ \G_{\pi}^{\text{full}} \geq \G^{\text{min}}_0\}\ .
$$
\end{corollary}
\begin{proof}
See Appendix~\ref{app:prooftoporder2}.
\end{proof}
This result is useful, for example, if the search over structures is performed in the space of permutations rather than in the space of DAGs \citep[e.g.][]{FriedmanKoller:MLJ03,teykol05,BuhlmannPetersErnest2013}.

\section{Algorithms} \label{sec:alg}
The theoretical results do not imply an algorithm for finitely many data that is either computationally or statistically efficient. In this section we propose an algorithm called RESIT that is based on independence-tests and two simple algorithms that make use of an independence score. We prove correctness of RESIT in the population case.

\subsection{Regression with Subsequent Independence Test (RESIT)} \label{sec:icml}
In practice, we are given i.i.d.\ data from the joint distribution and try to estimate the corresponding DAG. The following method is based on the fact that for each node $X_i$ the corresponding noise variable $N_i$ is independent of all non-descendants of $X_i$. In particular, for each sink node $X_i$ we have that $N_i$ is independent of $\X \setminus \{X_i \}$. 
We therefore propose an iterative procedure: in each step we identify and disregard a sink node. This is done by regressing each of the remaining variables on all other remaining variables and measuring the independence between the residuals and those other variables. The variable leading to the least dependent residuals is considered the sink node 
(Algorithm~\ref{alg:icml}, lines $4-13$). This first phase of the procedure yields a causal ordering or a fully connected DAG. In the second phase we visit every node and eliminate incoming edges until the residuals are not independent anymore, 
see Algorithm~\ref{alg:icml}, lines $15-22$.
The procedure can make use of any regression method and dependence measure, in this work we choose the $p$-value of the HSIC independence test \citep{Gretton2008} as a dependence measure. Under independence, \citet{Gretton2008} provide an asymptotically correct null distribution for the test statistic times sample size. (We use moment matching to approximate this distribution by a gamma distribution.) Since under dependence the test statistic is guaranteed to converge to a value different from zero, we know that the $p$-value converges to zero only for dependence. As a regression method we choose linear regression, gam regression (R package \texttt{mgcv}) or Gaussian process regression (R package \texttt{gptk}). 

Algorithm~\ref{alg:icml} is a slightly modified version of the one proposed in \citep{Mooij2009}. In this work, we always want to obtain a graph estimate; we thus consider the node with the least dependent residuals as being the sink node, instead of stopping the search when no independence hypothesis is accepted as in \citep{Mooij2009}.
\begin{algorithm}[h]
\caption{Regression with subsequent independence test (RESIT)}
\label{alg:icml}
\begin{algorithmic}[1]
  \STATE {\bfseries Input:} I.i.d. samples of a $p$-dimensional distribution on $(X_1, \ldots, X_p)$ \vspace{0.0cm}
   \STATE $S:=\{1, \ldots, p\}, \pi := [\ ]$ \vspace{0.15cm}
   \STATE PHASE 1: Determine causal order.
   \REPEAT
   \FOR{$k \in S$}
   \STATE Regress $X_k$ on $\{X_i\}_{i \in S \setminus \{k\}}$.
   \STATE Measure dependence between residuals and $\{X_i\}_{i \in S \setminus \{k\}}$.  
   \ENDFOR
   \STATE Let $k^*$ be the $k$ with the weakest dependence.
   \STATE $S:=S \setminus \{k^*\}$
   \STATE $\mathrm{pa}(k^*):=S$
   \STATE $\pi := [k^*, \pi]\qquad $  ($\pi$ will be the causal order, its last component being a sink)
   \UNTIL{$\#S=1$} \vspace{0.15cm}
   \STATE PHASE 2: Remove superfluous edges.
   \FOR{$k \in \{2, \ldots, p\}$}
   \FOR{$\ell \in \mathrm{pa}(\pi(k))$}
   \STATE Regress $X_{\pi(k)}$ on $\{X_i\}_{i \in \mathrm{pa}(\pi(k)) \setminus \{\ell\}}$.
   \IF{residuals are independent of $\{X_i\}_{i \in \{\pi(1), \ldots, \pi(k-1)\}}$} 
   \STATE $\mathrm{pa}(\pi(k)) := \mathrm{pa}(\pi(k)) \setminus \{\ell \}$ 
   \ENDIF
   \ENDFOR
   \ENDFOR
   \STATE {\bfseries Output:} $(\mathrm{pa}(1), \ldots, \mathrm{pa}(p))$
\end{algorithmic}
\end{algorithm}

Given that we have infinite data, a consistent non-parametric regression method and a perfect independence test (``independence oracle''), RESIT is correct.
\begin{theorem} \label{thm:algo}
Assume $\lawX = \law{X_1, \ldots, X_p}$ is generated by a restricted additive noise model with graph $\G_0$ and assume that $\lawX$ satisfies causal minimality with respect to $\G_0$. Then, RESIT used with a consistent non-parametric regression method and an independence oracle is guaranteed to find the correct graph $\G_0$ from the joint distribution $\lawX$.
\end{theorem}
\begin{proof}
See Appendix~\ref{app:proofalgo}
\end{proof}
RESIT performs $\mathcal{O}(p^2)$ independence tests, which is polynomial in the number of nodes. In phase 2 of the algorithm, superfluous edges are removed by variable selection. This is performed $\mathcal{O}(p)$ times. 
Both the independence test and the variable selection method may scale with the sample size, of course.
RESIT's polynomial behavior in $p$ may come as a surprise since problems in Bayesian network learning are often NP-hard \citep[e.g.][]{Chickering1996}.
Despite this theoretical guarantee, RESIT does not scale well to a high number of nodes. 
Since we cannot make use of an independence oracle in practice, we have to detect dependence between a random variable and a random vector from finitely many data. For high dimensions, this is a statistically hard problem that requires huge sample sizes. 

\subsection{Independence-Based Score} \label{sec:indbscore}
Searching for sink nodes makes the method described in Section~\ref{sec:icml} inherently asymmetric. Mistakes made in the first iterations propagate through the whole procedure. We therefore investigate the performance of independence-based score methods. Theorem~\ref{thm:mul} ensures that if the data come from a restricted additive noise model we can fit only one structure to the data. In order to estimate the graph structure we can test all possible DAGs and determine which DAG yields the most independent residuals. But even in the limit of infinitely many data we may find more than one DAG satisfying this constraint, some of which may not satisfy causal minimality.
We therefore propose to take a penalized independence score
\begin{equation} \label{eq:pis}
  \hat \G = \argmin_{\G} \sum_{i=1}^p \mathrm{DM}(\mathrm{res}_i^{\G,\mathrm{RM}}, \mathrm{res}_{-i}^{\G,\mathrm{RM}}) + \lambda \,\#(\text{edges})\,.
\end{equation}
Here, $\mathrm{res}_{i}$ are the residuals of node $X_i$, when regressing it on its parents; they depend on the graph $\G$ and on the regression method $\mathrm{RM}$. We denote the residuals of all variables except for $X_i$ by $\mathrm{res}_{-i}$ and $\mathrm{DM}$ denotes a measure of dependence. 
Note that variables $\B{N} = (N_1, \ldots, N_p)$ are jointly independent if and only if each $N_i$ is independent of $\B{N} \setminus \{N_i\}$, $i=1, \ldots, p$.
We do not prove (or claim) that the minimizer of~\eqref{eq:pis} is a consistent estimator for the correct DAG; we expect this to depend on the choice of $\mathrm{DM}$ and $\mathrm{RM}$ and $\lambda$.

As dependence measure we use {minus the logarithm of the $p$-values of an independence test based on the Hilbert Schmidt Independence Criterion HSIC} \citep{Gretton2008}. As regression methods we use linear regression, generalized additive models (gam) or Gaussian process regression.
For the regularization parameter $\lambda$ we propose to use $\log(0.05) - \log(0.01)$. This is a heuristic choice that is based on the following idea: we only allow for an additional edge if it allows the $p$-value to increase from $0.01$ to $0.05$ or, equivalently, by a factor of five. 
In practice, $p$-values estimated by bootstrap techniques or $p$-values that are smaller than computer precision can become zero and the logarithm becomes minus infinity. We therefore always consider the maximum of the computed $p$-value and $10^{-350}$. Although our choices seem to work well in practice, we do not claim that they are optimal.

\subsubsection{Brute-Force} \label{sec:bf}
For small graphs, we can solve equation~\eqref{eq:pis} by computing the score for all possible DAGs and choose the DAG with the lowest score.
Since the number of DAGs grows hyper-exponentially in the number of nodes, this method becomes quickly computationally intractable; e.g., for $p=7$, there are $1,138,779,265$ DAGs \citep{OEIS}. Nevertheless, we use this algorithm up to $p=4$ for comparison.


\subsubsection{Greedy DAG Search (GDS)} \label{sec:gds}
A strategy to circumvent the computational complexity of equation~\eqref{eq:pis} is to use greedy search algorithms \citep[e.g.,][]{Chickering2002}. At each step we are given a current DAG and score neighboring DAGs that are arranged in some order (see below). Here, all DAGs are called neighbors that can be reached by an edge reversal, addition or removal. Whenever a DAG has a better score than the current DAG, we stop scoring other neighbors and exchange the latter by the former. 
To obtain ``better'' steps, in each step we consider at least $p$ neighbors.
In order to reduce the running time of the algorithm, we do not score neighboring DAGs in a completely random order but start by adding or removing edges into nodes whose residuals are highly dependent on the other residuals instead. More precisely, we are randomly sorting the nodes, choosing each node one by one with a probability proportional to the reciprocal dependence measure of its residuals.
If all neighboring DAGs have a worse score than the current graph $G$, we nevertheless consider the best neighbor $H$. If $H$ has a neighbor with a better score than $G$, we continue with this graph. Otherwise we stop and output $G$ as the optimal graph. 
This is a simple version of tabu search \citep[e.g.][]{Koller2009} that is used to avoid local optima. 
This method is not guaranteed to find the best scoring graph. 

Code for the proposed methods is provided on the first and second authors' homepages.

\section{Experiments} \label{sec:exp}

\subsection{Experiments on Synthetic Data}
For varying sample size $n$ and number of variables $p$ we compare the described methods.  
Given a value of $p$, we randomly choose an ordering of the variables with respect to the uniform distribution and include each of the $p(p-1)/2$ possible edges with a probability of $2/(p-1)$. This results in an expected number of $p$ edges and can be considered as a (modestly) sparse setting.
For a linear and a nonlinear setting we report the average structural Hamming distance \citep{Acid2003, Tsamardinos2006} 
to the true directed acyclic graph and to the true completed partially directed acyclic graph over $100$ simulations.
The structural Hamming distance (SHD) between two partially directed acyclic graphs
counts how many edge types do not coincide. Estimating a non-edge or a directed edge instead of an undirected edge, for example, contributes an error of one to the overall distance.
We also report analogous results for the structural intervention distance (SID), which has recently been proposed \citep{Peters2013sid}. 
Given the estimated graph we can infer the intervention distribution $p(X_j\given do(X_i=x_i))$ by the parent adjustment \eqref{eq:parentadj}.
We call a pair of nodes $(X_i,X_j)$ \emph{good} if the intervention distribution $p(X_j\given do(X_i=x_i))$ inferred from the estimated DAG coincides with the intervention distribution inferred from the correct DAG for all observational distributions $\lawX$. 
The SID counts the number of pairs that are not good. Some methods output a Markov equivalence class instead of a single DAG. Different DAGs within such a class lead to  different intervention distribution and thus different SIDs. In that case, we therefore provide the smallest and largest SID attained by members within the Markov equivalence class. 
As the SHD, the SID is a purely structural measure that is independent of any distribution. The rationale behind the new measure is that a reversed edge in the estimated DAG leads to more false causal effects than an additional edge does. The SHD, however, weights both errors equally.

We compare the 
greedy DAG search (GDS), brute-force (BF), regression with subsequent independence test (RESIT), linear non-Gaussian additive models (LINGAM), the PC algorithm (PC) with partial correlation and significance level $0.01$ and greedy equivalence search (GES), see Sections~\ref{sec:gds}, \ref{sec:bf}, \ref{sec:icml}, \ref{sec:lin}, \ref{sec:ibm} and \ref{sec:scorebm}, respectively. We also compare them with the conservative PC algorithm (CPC), suggested by \citet{Ramsey2006}, and random guessing (RAND). The latter chooses a random DAG with edge inclusion probability uniformly chosen between zero and one. Its estimate does not depend on the data.

\subsubsection{Linear Structural Equation Models}
We first consider a linear setting as in equation~\eqref{eq:lingam}, where the coefficients $\beta_{jk}$ are uniformly chosen from $[-2,-0.1] \cup [0.1,2]$ and the noise variables $N_j$ are independent and distributed according to $K_j \cdot \mathrm{sign}(M_j)\cdot |M_j|^{\alpha_j}$ with $M_j \iid \mathcal{N}(0,1)$, $K_j \iid \mathcal{U}([0.1,0.5])$ and $\alpha_j \iid \mathcal{U}([2,4])$. 
The top box plot in Figure~\ref{fig:linear} compares the SHD of the estimated structure to the correct DAG for $p=4$ and $n=100$. The brute-force method performs best, which indicates that the score function in equation~\eqref{eq:pis} is a sensible choice for small graphs. Greedy DAG search performs almost equally well, it does not encounter many local optima in this setting.
The constraint-based methods and greedy equivalent search perform worse.
Comparing SID leads to the same conclusion (Figure~\ref{fig:linear}, bottom).
\begin{figure}[h]
\begin{center}
\includegraphics[width=0.8\textwidth]{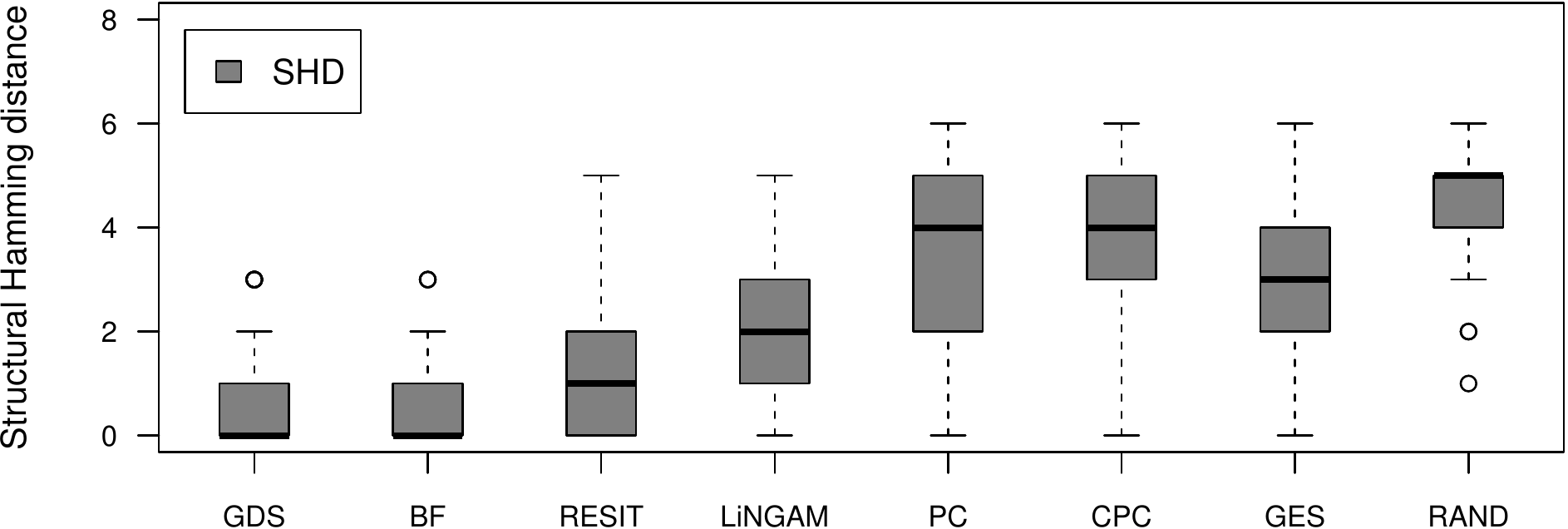}
\includegraphics[width=0.8\textwidth]{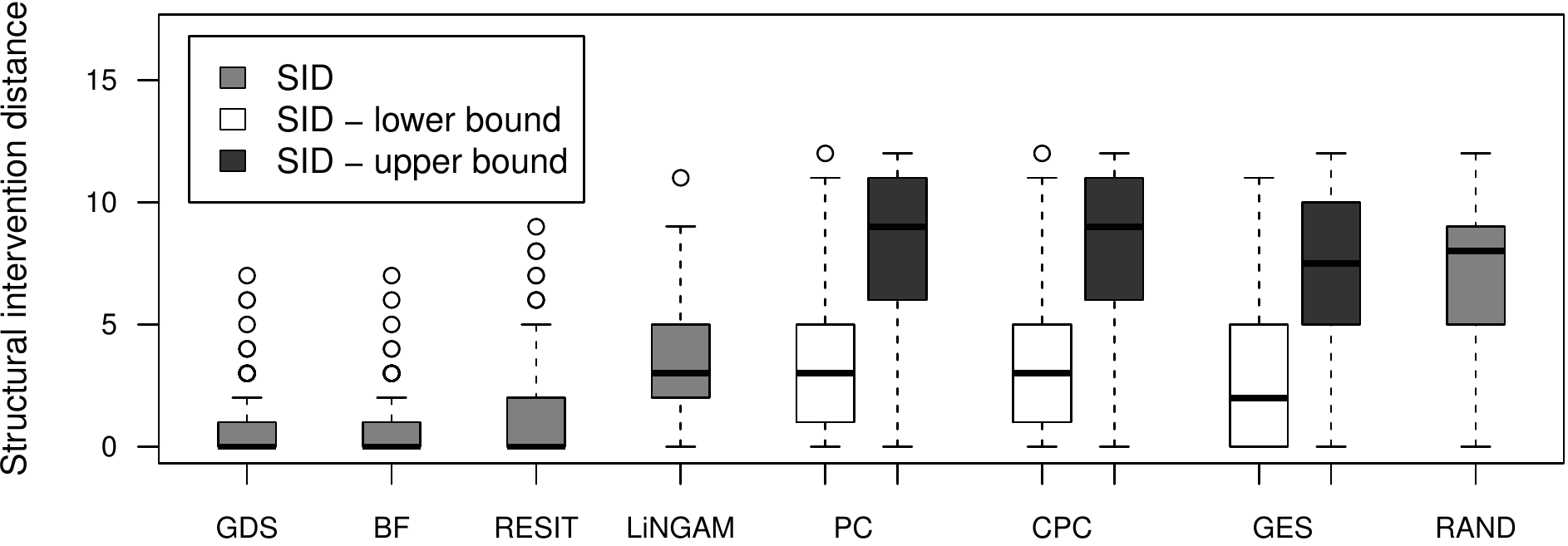}
\end{center}
\caption{Box plots of the SHD between the estimated structure (either DAG or CPDAG) and the correct DAG for $p=4$ and $n=100$ for linear non-Gaussian SEMs (top). The SID is computed between the correct DAG and the estimated DAG (bottom). Some methods estimate only the Markov equivalence class. We then compute the SID to the ``best'' and to the ``worst'' DAG within the equivalence class; therefore a lower and an upper bound are shown.}
\label{fig:linear}
\end{figure}

Tables~\ref{tab:shd} and~\ref{tab:sid} provide summaries for $p \in \{4,15\}$ and $n \in \{100,500\}$. We additionally show distances of the estimated CPDAGs to the true CPDAGs. Therefore, if methods output a DAG instead of a CPDAG, this DAG is transformed into the CPDAG of the corresponding Markov equivalence class.
For $p=4$ and $n=500$, GDS and brute force find almost always the correct graph ($86$ and $90$ out of $100$). RESIT and LiNGAM still perform much better than the PC methods and GES.
For $p=15$, the performance of RESIT (and GES) in relation to the other methods seems to be better when evaluating SID compared to evaluating the SHD. This indicates that the pruning (and penalization of the number of edges) does not work perfectly.
The brute-force method is not applicable to $p=15$.
\begin{table}[h]
\caption{Linear SEMs: SHD between the estimated structure and the correct DAG and SHD between the estimated CPDAG to the correct CPDAG; 
for both the average and the standard deviation over $100$ experiments are shown (best averages are highlighted).
 } 
 
\begin{center}
{\scriptsize
\begin{tabular}{c||c|c|c|c|c|c|c|c} 
& GDS & BF & RESIT & LiNGAM & PC & CPC & GES & RAND \\ \hline \hline
&\multicolumn{8}{c}{$p=4, n=100$}\\ \hline 
DAG & $0.7 \pm 0.9$& $0.6 \pm 0.8$ \cellcolor{lightgray} &$ 1.2 \pm 1.3$&$ 1.9 \pm 1.2$&$ 3.5  \pm 1.5$&$ 3.6  \pm 1.4$&$ 3.1 \pm 1.7$&$ 4.4 \pm 1.0$ \\ 
\hhline{-|-|-|-|-|-|-|-|-}
CPDAG & $1.1 \pm 1.5 $&$ 0.9 \pm 1.4 $ \cellcolor{lightgray}&$ 1.5  \pm 1.7$&$ 2.4  \pm 1.5$&$ 2.4  \pm 1.7$&$ 2.3  \pm 1.6$&$ 2.0  \pm 2.0$&$ 4.3 \pm 1.4$ \\ 
\hline\hline 
&\multicolumn{8}{c}{$p=4, n=500$}\\ \hline
DAG & $0.2  \pm 0.6$&$ 0.1  \pm 0.3$ \cellcolor{lightgray}&$ 0.6  \pm 0.8$&$ 0.5  \pm 0.8$&$ 3.1  \pm 1.4$&$ 3.2  \pm 1.4$&$ 2.9  \pm 1.6$&$ 4.1  \pm 1.2$ \\ 
\hline
                           CPDAG & $0.3  \pm 0.9$&$ 0.2  \pm 0.5$ \cellcolor{lightgray}&$ 0.9  \pm 1.3$&$ 0.8  \pm 1.2$&$ 1.9  \pm 1.8$&$ 1.6  \pm 1.7$&$ 1.6  \pm 1.9$&$ 3.9 \pm 1.4$\\ 
\hline\hline
&\multicolumn{8}{c}{$p=15, n=100$}\\\hline
  DAG & $12.2 \pm 5.3$&$ - $&$ 25.2 \pm 8.3 $&$ 11.1  \pm 3.7$ \cellcolor{lightgray}&$ 13.0  \pm 3.6$&$ 13.7  \pm 3.7$&$ 12.7  \pm 4.2$&$ 57.4 \pm 26.4$ \\ 
\hline
                              CPDAG                     & $13.2  \pm 5.4$&$ - $&$ 27.0  \pm 8.5$&$ 12.4  \pm 3.9$ &$ 10.7 \pm 3.5 $\cellcolor{lightgray}&$ 10.8  \pm 3.8$&$ 12.4  \pm 4.9$&$ 58.5 \pm 27.1$\\ 
\hline\hline
&\multicolumn{8}{c}{$p=15, n=500$}\\ \hline
  DAG      & $6.1 \pm 6.4 $&$ - $&$ 51.2  \pm 17.8$&$ 3.4 \pm 2.8$ \cellcolor{lightgray}&$ 10.2  \pm 3.8$&$ 10.8 \pm 4.2 $&$ 8.7 \pm 4.6 $&$ 57.6 \pm 24.2$\\ 
\hline
                              CPDAG & $6.8 \pm 6.9$&$ - $&$ 54.5  \pm 18.5$&$ 4.5  \pm 3.8$ \cellcolor{lightgray}&$ 8.2  \pm 4.6$&$ 7.5 \pm 4.4 $&$ 7.1 \pm 5.6 $&$ 58.9 \pm 25.0$
\end{tabular}
}
\label{tab:shd}
\end{center}
\end{table}

\begin{table}[h]
\caption{Linear SEMs: SID to the correct DAG; the table shows average and standard deviation over $100$ experiments.} 
\begin{center}
{\scriptsize
\begin{tabular}{c|c|c|c|c|c|c|c} 
GDS & BF & RESIT & LiNGAM & PC & CPC & GES & RAND \\ \hline \hline
\multicolumn{8}{c}{$p=4, n=100$}\\ \hline
\multirow{2}{*}{$1.0 \pm 1.5$}&\cellcolor{lightgray}&\multirow{2}{*}{$1.5 \pm 2.2$}&\multirow{2}{*}{$3.3 \pm 2.1$}&$3.4 \pm 2.9$&$3.2 \pm 2.7$&$2.9 \pm 3.3$ & \multirow{2}{*}{$7.0 \pm 2.8$} \\ 
&\multirow{-2}{*}{$0.8 \pm 1.4$} \cellcolor{lightgray}&&&$8.0 \pm 3.2$&$8.5 \pm 3.2$&$7.2 \pm 3.5$& \\ 
\hline \hline
\multicolumn{8}{c}{$p=4, n=500$}\\ \hline
\multirow{2}{*}{$0.2 \pm 0.7$}&\cellcolor{lightgray}&\multirow{2}{*}{$0.3 \pm 1.0$}&\multirow{2}{*}{$0.9 \pm 1.4$}&$2.8 \pm 3.1$&$2.3 \pm 2.7$&$2.1 \pm 2.9$ & \multirow{2}{*}{$6.3 \pm 2.8$} \\ 
&\multirow{-2}{*}{$0.1 \pm 0.4$}\cellcolor{lightgray}&&&$7.4 \pm 3.4$&$7.6 \pm 3.3$&$6.9 \pm 3.6$&\\ 
\hline \hline
\multicolumn{8}{c}{$p=15, n=100$}\\ \hline
\cellcolor{lightgray}&\multirow{2}{*}{$-$}&\multirow{2}{*}{$35.3 \pm 21.2$}&\multirow{2}{*}{$45.1 \pm 24.1$}&$36.5 \pm 21.3$&$32.5 \pm 20.2$&$26.5 \pm 18.3$ \cellcolor{lightgray}& \multirow{2}{*}{$55.6 \pm 27.1$}\\
\multirow{-2}{*}{$32.3 \pm 24.1$}\cellcolor{lightgray}&&&&$63.7 \pm 30.3$&$66.4 \pm 31.5$&$37.6 \pm 20.6$ \cellcolor{lightgray}& \\ 
\hline \hline
\multicolumn{8}{c}{$p=15, n=500$}\\ \hline
\cellcolor{lightgray}&\multirow{2}{*}{$-$}&\multirow{2}{*}{$18.1 \pm 13.8$}&\multirow{2}{*}{$14.2 \pm 14.6$}&$33.6 \pm 29.5$&$23.2 \pm 19.8$&$18.1 \pm 21.4$ & \multirow{2}{*}{$57.5 \pm 34.1$}\\
\multirow{-2}{*}{$12.6 \pm 16.3$}\cellcolor{lightgray}&&&&$55.0 \pm 32.9$&$55.6 \pm 32.4$&$31.6 \pm 22.2$&
\end{tabular}
}
\label{tab:sid}
\end{center}
\end{table}

\subsubsection{Nonlinear Structural Equation Models}
We also sample data from nonlinear SEMs. We choose an additive structure as in equation~\eqref{eq:semaddadd} and sample the functions from a Gaussian process with bandwidth one.
The noise variables $N_j$ are independent and normally distributed with a uniformly chosen variance. 
Tables~\ref{tab:shdnl} and~\ref{tab:sidnl}
show summaries for $p \in \{4,15\}$ and $n \in \{100,500\}$. 
We cannot run the brute-force method on data sets with $p=15$. 
For $p=4$, we have a similar situation as in Figure~\ref{fig:linear} with GDS and the BF method outperforming all others (RESIT performing a bit worse). Remarkably, for $p=15$ and $n=100$, a lot of the methods do not perform much better than random guessing when comparing the SID. The estimated CPDAG of the constraint-based methods can have very different lower and upper bounds for SID. This means that some DAGs within the equivalence class perform much better than others. (The methods do not propose any particular DAG, they treat all DAGs within the class equally.) 
\begin{table}[t]
\caption{Nonlinear SEMs: SHD between the estimated structure and the correct DAG and SHD between the estimated CPDAG to the correct CPDAG; 
for both the average and the standard deviation over $100$ experiments are shown. 
 } 
 
\begin{center}
{\scriptsize
\begin{tabular}{c||c|c|c|c|c|c|c|c} 
& GDS & BF & RESIT & LiNGAM & PC & CPC & GES & RAND \\ \hline \hline
&\multicolumn{8}{c}{$p=4, n=100$}\\ \hline
DAG & $1.5 \pm 1.4$&$ 1.0 \pm 1.0 $\cellcolor{lightgray}&$ 1.7  \pm 1.3$&$ 3.5 \pm 1.2 $&$ 3.5 \pm 1.5 $&$ 3.8 \pm 1.4 $&$ 3.5 \pm 1.3 $&$ 4.0 \pm 1.3$\\ 
\hhline{-|-|-|-|-|-|-|-|-}
                              CPDAG &                    $1.7 \pm 1.7 $&$ 1.2 \pm 1.4 $\cellcolor{lightgray}&$ 2.0 \pm 1.6 $&$ 3.0 \pm 1.4 $&$ 2.9 \pm 1.5 $&$ 2.7 \pm 1.4 $&$ 3.4 \pm 1.7 $&$ 3.9 \pm 1.4$ \\ 
\hline\hline 
&\multicolumn{8}{c}{$p=4, n=500$}\\ \hline                              DAG &                         $0.5 \pm 0.9$&$ 0.3 \pm 0.5 $\cellcolor{lightgray}&$ 0.8 \pm 0.9 $&$ 3.7 \pm 1.2 $&$ 3.5 \pm 1.5 $&$ 3.8 \pm 1.5 $&$ 3.3 \pm 1.5 $&$ 4.1 \pm 1.2$ \\ \hhline{-|-|-|-|-|-|-|-|-}
                              CPDAG &                    $0.6 \pm 1.1 $\cellcolor{lightgray}&$ 0.6 \pm 1.0 $\cellcolor{lightgray}&$ 1.0 \pm 1.3 $&$ 3.0 \pm 1.7 $&$ 3.1 \pm 1.9 $&$ 2.8 \pm 1.8 $&$ 3.4 \pm 1.9 $&$ 3.8 \pm 1.6$ \\ 
\hline\hline 
&\multicolumn{8}{c}{$p=15, n=100$}\\ \hline                              DAG& $14.3 \pm 4.9$&$ - $&$ 15.4 \pm 5.7$&$ 15.4 \pm 3.6$&$ 14.2 \pm 3.5 $\cellcolor{lightgray}&$ 15.5 \pm 3.6 $&$ 24.8 \pm 6.3 $&$ 56.8 \pm 24.1$ \\ 
\hhline{-|-|-|-|-|-|-|-|-}
                              CPDAG                 & $15.1 \pm 5.4 $&$ - $&$ 16.5 \pm 5.9 $&$ 15.3 \pm 4.0 $&$ 13.3 \pm 3.6 $\cellcolor{lightgray}&$ 13.3  \pm 4.0$ \cellcolor{lightgray}&$ 26.4  \pm 6.5$&$ 58.0 \pm 24.7$\\ 
\hline\hline 
&\multicolumn{8}{c}{$p=15, n=500$}\\ \hline
 DAG      & $13.0 \pm 8.4 $&$ - $&$ 10.1 \pm 5.7 $\cellcolor{lightgray}&$ 21.4 \pm 6.9 $&$ 13.9 \pm 4.5 $&$ 15.1 \pm 4.8 $&$ 26.8 \pm 8.5 $&$ 56.1 \pm 26.8$\\ \hhline{-|-|-|-|-|-|-|-|-}
                              CPDAG & $14.2 \pm 9.2 $&$ - $&$ 11.3 \pm 6.3 $\cellcolor{lightgray}&$ 21.1 \pm 7.3 $&$ 13.7 \pm 4.9 $&$ 13.4 \pm 5.1 $&$ 28.6 \pm 8.8 $&$ 57.0 \pm 27.3$\\ 
\end{tabular}
}
\label{tab:shdnl}
\end{center}
\end{table}
\begin{table}[t]
\caption{Nonlinear SEMs: SID to the correct DAG; the table shows average and standard deviation over $100$ experiments.}  
\begin{center}
{\scriptsize
\begin{tabular}{c|c|c|c|c|c|c|c} 
GDS & BF & RESIT & LiNGAM & PC & CPC & GES & RAND \\ \hline \hline
\multicolumn{8}{c}{$p=4, n=100$}\\ \hline
\multirow{2}{*}{$2.0 \pm 2.5$}&\cellcolor{lightgray}&\multirow{2}{*}{$2.0 \pm 1.9$}&\multirow{2}{*}{$8.2 \pm 2.8$}&$4.7 \pm 3.2$&$4.3 \pm 2.7$&$4.7 \pm 3.2$ & \multirow{2}{*}{$6.3 \pm 3.1$}\\ 
&\multirow{-2}{*}{$1.4 \pm 1.7$}\cellcolor{lightgray}&&&$7.8 \pm 3.4$&$8.5 \pm 3.2$&$7.2 \pm 3.2$&\\ 
\hline \hline
\multicolumn{8}{c}{$p=4, n=500$}\\ \hline
\multirow{2}{*}{$0.6 \pm 1.8$}&\cellcolor{lightgray}&\multirow{2}{*}{$0.9 \pm 1.3$}&\multirow{2}{*}{$8.0 \pm 2.8$}&$4.3 \pm 3.7$&$3.7 \pm 3.3$&$3.6 \pm 3.0$ & \multirow{2}{*}{$6.6 \pm 3.4$}\\ 
&\multirow{-2}{*}{$0.2 \pm 0.8$}\cellcolor{lightgray}&&&$7.3 \pm 3.2$&$8.1 \pm 3.2$&$6.5 \pm 3.3$&\\ 
\hline \hline
\multicolumn{8}{c}{$p=15, n=100$}\\ \hline
\multirow{2}{*}{$50.6 \pm 25.3$}&\multirow{2}{*}{$-$}&\cellcolor{lightgray}&\multirow{2}{*}{$65.0 \pm 28.3$}&$49.7 \pm 24.6$&$40.4 \pm 21.6$\cellcolor{lightgray}&$49.0 \pm 27.3$ & \multirow{2}{*}{$60.0 \pm 29.9$}\\
&&\multirow{-2}{*}{$44.4 \pm 23.9$}\cellcolor{lightgray}&&$68.6 \pm 31.5$&$76.7 \pm 32.8$\cellcolor{lightgray}&$53.6 \pm 28.9$& \\ 
\hline \hline
\multicolumn{8}{c}{$p=15, n=500$}\\ \hline
\multirow{2}{*}{$35.9 \pm 26.8$}&\multirow{2}{*}{$-$}&\cellcolor{lightgray}&\multirow{2}{*}{$67.3 \pm 28.1$}&$49.9 \pm 29.0$&$36.4 \pm 22.1$&$40.2 \pm 23.3$ & \multirow{2}{*}{$58.9 \pm 27.8$}\\
&&\multirow{-2}{*}{$24.6 \pm 18.6$}\cellcolor{lightgray}&&$60.3 \pm 31.0$&$70.3 \pm 34.6$&$44.6 \pm 24.0$&
\end{tabular}
}
\label{tab:sidnl}
\end{center}
\end{table}
\begin{figure}[h]
\begin{center}
\includegraphics[width=0.8\textwidth]{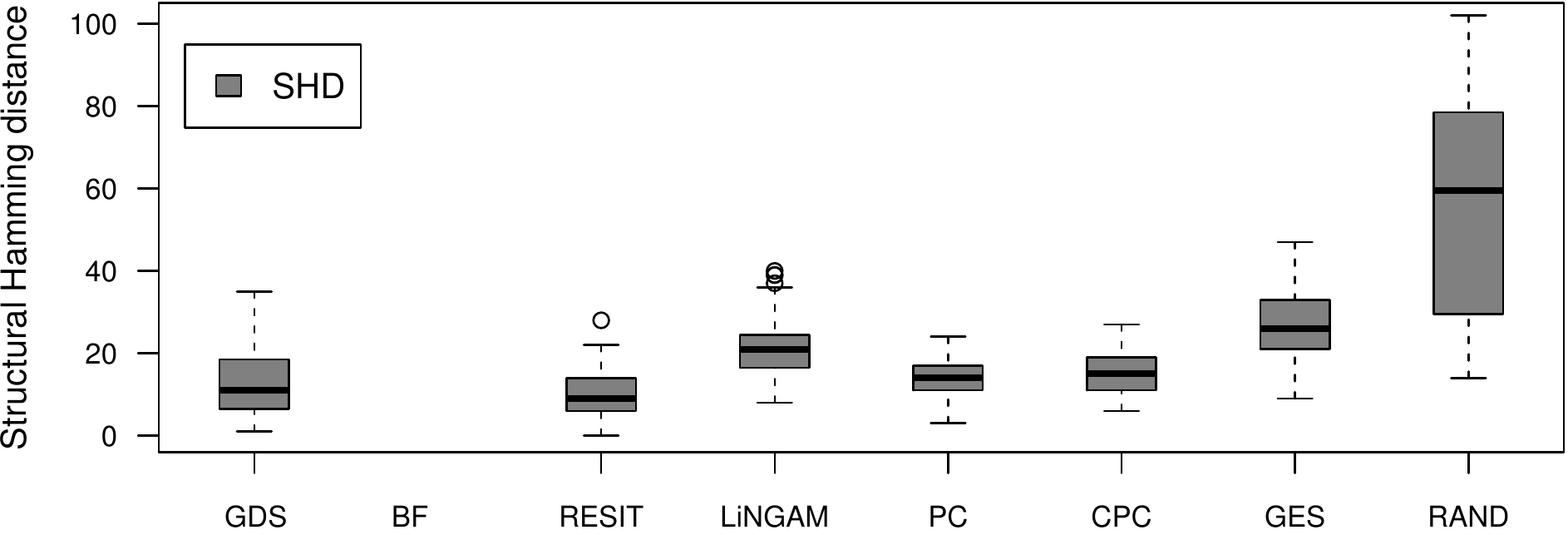}
\includegraphics[width=0.8\textwidth]{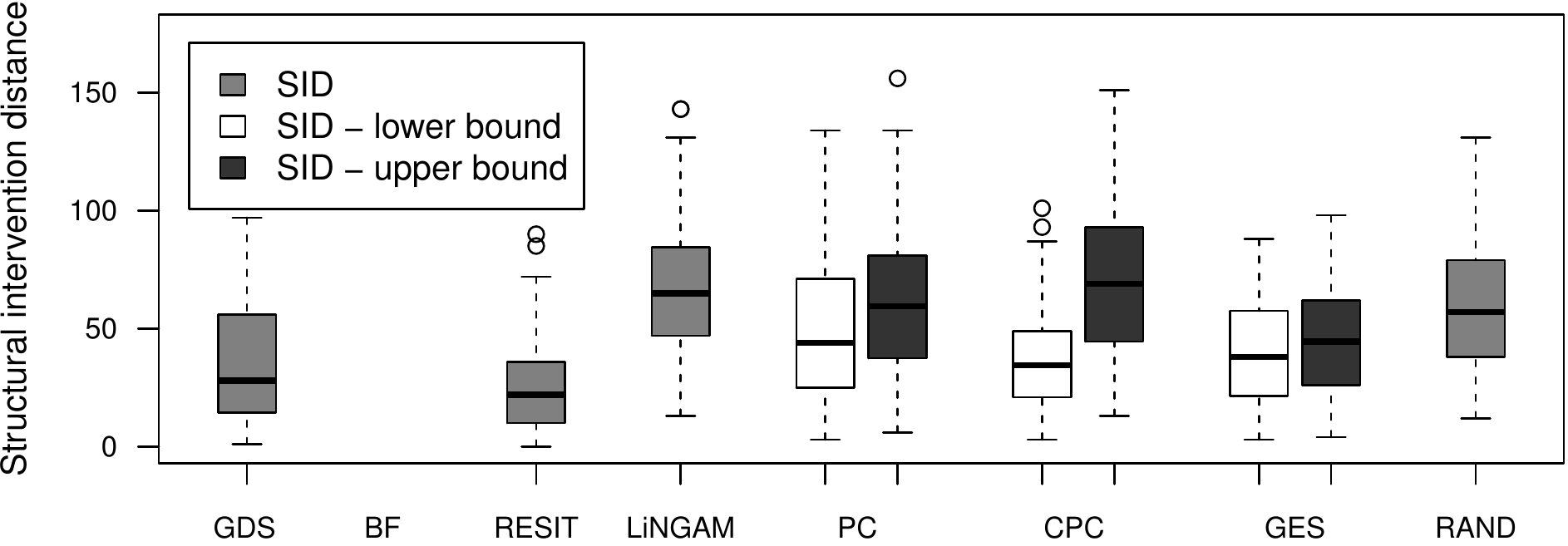}
\end{center}
\caption{Similar to Figure~\ref{fig:linear}: box plots of the SHD between estimated structure and correct DAG (top) and
box plots of the SID to the correct DAG (bottom) for $p=15$, $n=500$ and nonlinear Gaussian SEMs.}
\label{fig:nonlinear}
\end{figure}
Figure~\ref{fig:nonlinear} shows box plots of SHD and SID for the special case $p=15$ and $n=500$.
This time, RESIT perform slightly better than all other methods. It makes use of the nonlinearity of the structural equations. Again, the high SHD for GES indicates that the estimate probably contains too many edges (since its SID is better than the one for the PC methods).\\

In conclusion, for $p=4$, the brute force method works best for both linear and nonlinear data. Roughly speaking, for $p=15$, LiNGAM and GDS work best in the linear non-Gaussian setting and RESIT works best for nonlinear data. If one does not know whether the data are linear or nonlinear, GDS provides an alternative that works reasonably well in both settings.

\subsection{Altitude, Temperature and Duration of Sunshine}
We consider recordings of average temperature $T$, average duration of sunshine $DS$ and the altitude $A$ at $349$ German weather stations \citep{DWD}. Figure~\ref{fig:alt} shows scatter plots of all pairs. 
\begin{figure}[h]
\begin{center}
\includegraphics[width=0.31\textwidth]{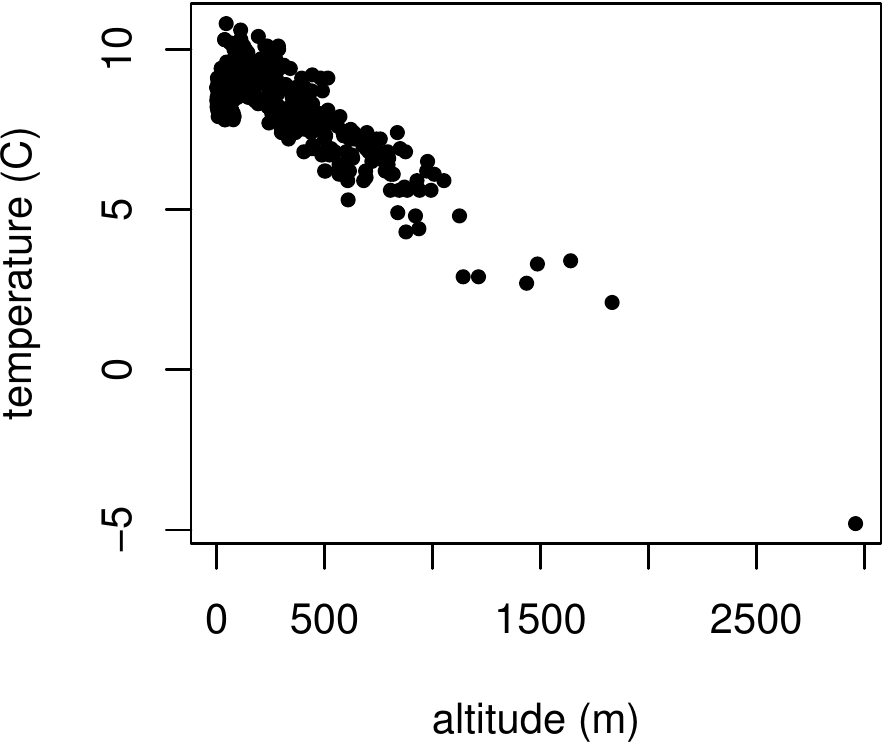}
\hspace{0.02\textwidth}
\includegraphics[width=0.31\textwidth]{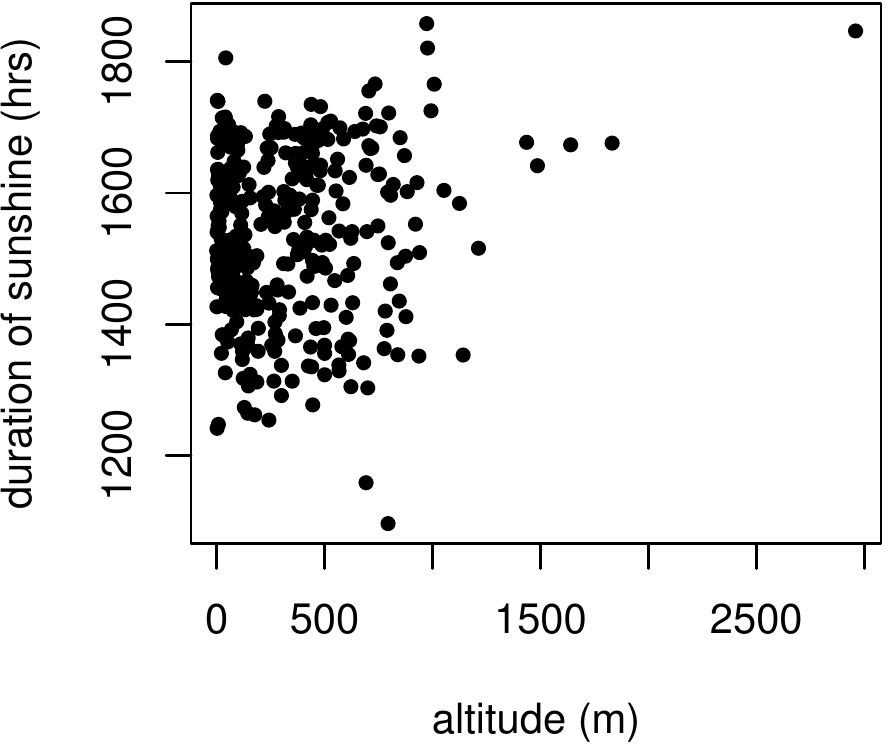}
\hspace{0.02\textwidth}
\includegraphics[width=0.31\textwidth]{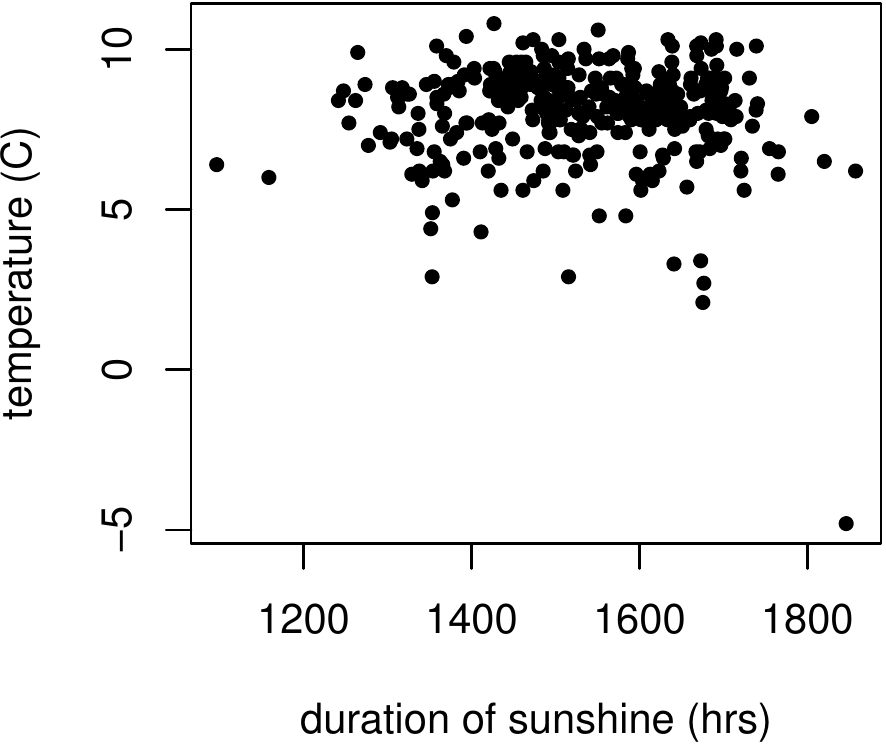}
\end{center}
\caption{Scatter plots of the three pairs, altitude, temperature and duration of sunshine.}
\label{fig:alt}
\end{figure}
LiNGAM estimates $T \rightarrow A$, PC and CPC estimate $T \rightarrow A \leftarrow DS$, GES estimates a fully connected DAG. The brute-force estimate with linear regression obtains a score of $103.6$. Since we are taking the logarithm to base $10$ in equation~\eqref{eq:pis}, we see that the model does not fit the data well. More sensible seems the gam regression, for which both GDS and brute-force output the DAG
$T \leftarrow A \rightarrow DS$ and $T \rightarrow DS$, which receives a score of $5.9$. 
Also RESIT outputs this DAG. Although there might be a feedback between duration of sunshine and temperature through the generation of clouds, we believe that the link from sunshine to temperature should be stronger. In fact, the corresponding DAG
$T \leftarrow A \rightarrow DS$ with $T \leftarrow DS$ receives the second best score.
Furthermore, these data may be confounded by geographical location. Together with the possible feedback loop and a possible deviation from additive noise models this might be the reason why we do not obtain clear independence of the residuals: the HSIC between the residuals of temperature and the two others leads to a $p$-value of $0.012$ (the other two $p$-values are both about $0.12$). In practice, we often expect some violations of the model assumptions. This example, however, indicates that it may still possible to obtain reasonable estimates of the underlying causal structure if the violations are not too strong.

\subsection{Cause-Effect Pairs}
We have tested the performance of additive noise models on a collection of various cause-effect pairs,
an extended version of the ``Cause-effect pairs'' dataset described in \citep{CauseEffectPairs}.
As of this writing, this dataset consists of observations of $86$ different pairs of variables from 
various domains. The task is to infer which variable is the cause and which variable the effect,
for each of the pairs. For example, one of the pairs consists of 349 measurements of
altitude and temperature taken at different weather stations in Germany \citep{DWD}, the same data
as considered in the previous subsection. It should be obvious that here the altitude is the cause,
and the temperature is the effect. The complete dataset and a more detailed
description of each pair can be obtained from
\url{http://webdav.tuebingen.mpg.de/cause-effect}.


For each pair of variables $(X_i,Y_i)$, with $i=1,\dots,86$, we test the two possible additive noise
models that correspond with the two different possible causal directions, $X_i
\to Y_i$ and $Y_i \to X_i$. For both directions, we estimate the functional
relationship by performing Gaussian Process regression using the GPML toolbox
\citep{RasmussenNickisch10}. We use the expected value of the Gaussian Process given the observations as an estimate of the functional dependence between the cause and the effect. The goodness-of-fit is then evaluated by testing independence of the residuals and the inputs. Here, we use the HSIC as an independence test and approximate the null distribution with a gamma distribution in order to obtain $p$-values \citep{Gretton2005JMLR}. We thus obtain two
$p$-values for each pair, one for each possible causal direction (where a high $p$-value corresponds to not rejecting independence, i.e., not rejecting the causal model). We then rank the pairs according to the highest of the two $p$-values of the pair. Using this ranking, we can make decisions for only a subset of the pairs, starting with the pair for which the highest of the two $p$-values is the largest among all pairs (we say these pairs have a high rank). In this way we trade off accuracy, i.e., percentage of correct decisions, versus the amount of decisions taken. 

Five of the pairs have multivariate $X_i$ or $Y_i$, and we did not include those in the analysis for convenience. Furthermore, not all the pairs are independent; for example, life expectancy versus
latitude occurs more than once, but measurements were done in different years
and for different gender. We therefore assigned weights to the cause-effect pairs to compensate for this when calculating the accuracy and decision rate. For example, the pair life expectancy versus latitude appears eight times (for different combinations of gender and year), hence each of these pairs is weighted down with the factor $1/8$; on the other hand, the pair altitude vs.\ temperature at weather stations occurs only once, and therefore gets weight $1$. Denoting the weight of each pair with $w_i$, the ``effective'' number of pairs becomes $\sum_{i=1}^{86} w_i = 68$.
If the set of highest-ranked pairs is denoted $\C{I}$, and the set of correct decisions is denoted $\C{C}$, then the \emph{accuracy} (fraction of correct decisions) is defined as
$$\text{accuracy} = \frac{\sum_{i\in \C{I} \cap \C{C}} w_i}{\sum_{i\in \C{I}} w_i}$$
and the \emph{decision rate} (fraction of decisions taken) is defined as
$$\text{decision rate} = \frac{\sum_{i\in \C{I}} w_i}{\sum_{i=1}^{86} w_i}.$$

The results are plotted in Figure \ref{fig:cep}. It shows the accuracy (dark blue line) as a function of the decision rate, together with confidence intervals (light blue regions).
The amount of cause-effect pairs from which the accuracy can be estimated decreases proportionally to the decision rate; the accuracies reported for low decision rates therefore have higher uncertainty than the accuracies reported for high decision rates. For each decision rate, we have plotted the 68\% and 95\% confidence intervals for the estimated success probability assuming a binomial distribution using the Clopper-Pearson method. If for a given decision rate, the 95\% confidence region lies above the line at 50\%, the method performs significantly better than random guessing (for that decision rate). For example, if we take a decision for all pairs, $72 \pm 6 \%$ of the decisions are correct, significantly more than random guessing. If we only take the $20\%$ most confident decisions, all of
them are correct, again significantly more than random guessing.
\begin{figure}[h]
\centerline{\includegraphics[width=0.7\textwidth]{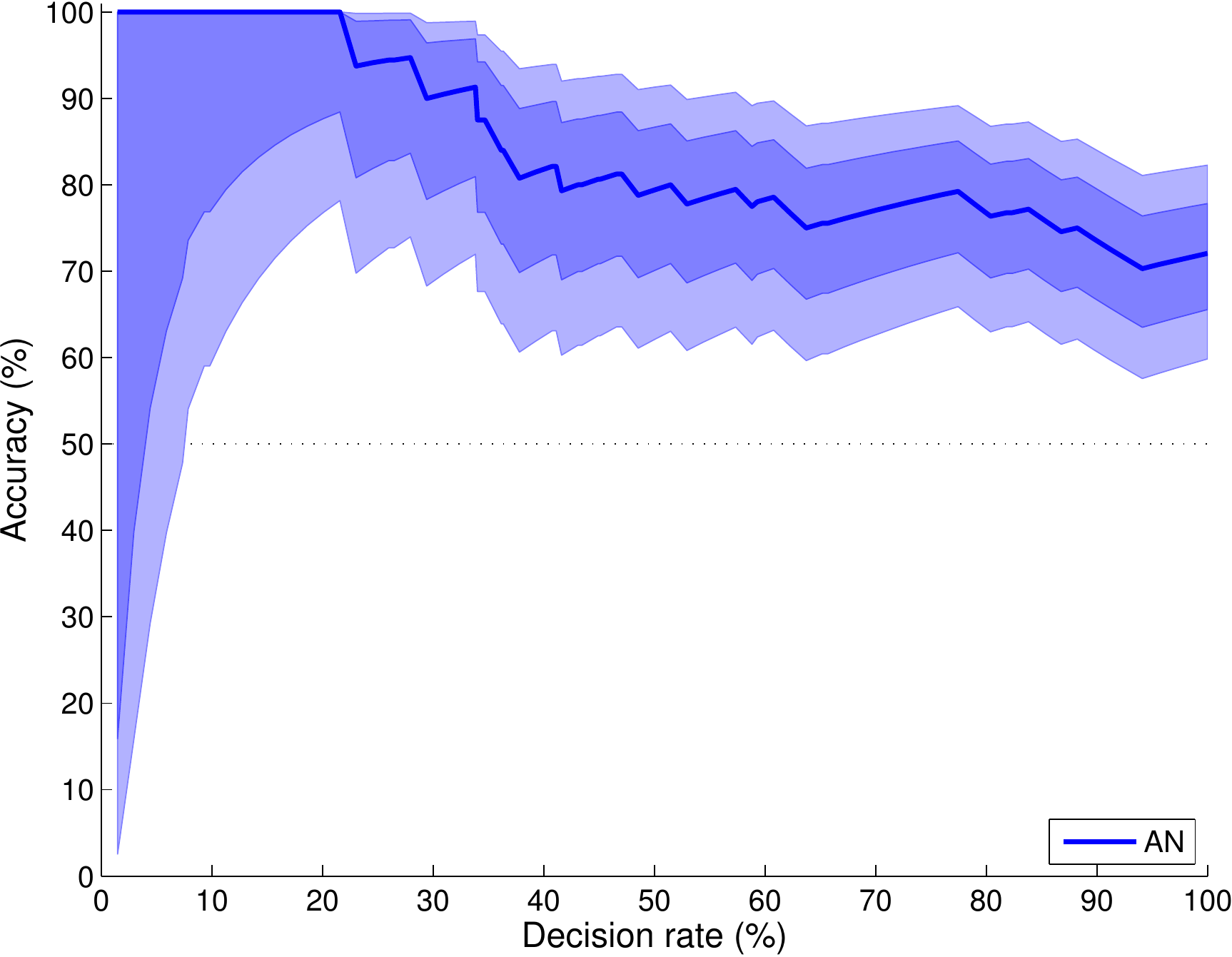}}
\caption{\label{fig:cep}Results of the additive noise method on 86 cause-effect pairs, showing estimated accuracy, 68\% and 95\% confidence intervals for each decision rate.}
\end{figure}



\section{Discussion and Future Work} \label{sec:con}
Apart from a few exceptions 
we can identify the directed acyclic graph from
a bivariate distribution
that has been generated by a
structural equation model with continuous additive noise.
Such an identifiability in the bivariate case generalizes under mild assumptions to identifiability in the multivariate case (i.e., graphs with more than two variables).
This can be beneficial for the field of causal inference: 
if the true data generating process can be represented by a restricted structural equation model like additive noise models, the causal graph can be inferred from the joint distribution.
We believe that formulating the problem using structural equation models rather than graphical models made it easier to state and exploit the assumption of additive noise.
While the language of graphical models allow us to define some notion connecting a graph to the distribution (e.g., faithfulness), SEMs allow us to impose specific restrictions on the possible functional relationships between nodes and its children. This is closer in spirit to a machine learning approach where properties of function classes play a crucial role in the estimation.

Both artificial and real data sets indicate that methods based on restricted structural equation models can outperform traditional constraint-based methods.
We have proposed a score that reflects the independence of residuals.
Although the score seems to be suitable to detect the correct graph structure, 
it remains unclear how to find the best scoring DAG when an exhaustive search is infeasible.
One possibility is to search this space by greedily choosing best-scoring neighbors. 
Multiple random initializations may decrease the chance that the greedy DAG search gets stuck in local optima by the additional cost of computational complexity.
We further believe that the proposed score may benefit from an extended version of HSIC that is able to estimate mutual independence instead of pairwise independence.
Recently, \citet{Nowzohour2013} have suggested a penalized likelihood based score for bivariate models. They estimate the noise distribution and use the BIC for penalization. In principle this idea can again be combined with a brute-force search as in Section~\ref{sec:bf} or a greedy DAG search as in Section~\ref{sec:gds}.
Making the methods applicable to larger graphs ($p>20$) remains a major challenge. 
Also, studying the statistical properties of the methods (for example, establishing consistency) is
an important task for future research.

\section*{Acknowledgements}
We thank Peter B\"uhlmann and Markus Kalisch for helpful discussions and Patrik Hoyer for the collaboration initiating the idea of nonlinear additive noise models \citep{Hoyer2008}. 
The research leading to these results has received funding from the People Programme (Marie Curie Actions) of the European Union's Seventh Framework Programme (FP7/2007-2013) under REA grant agreement no 326496.
JM was supported by NWO, the Netherlands Organization for Scientific Research (VENI grant 639.031.036).
We thank the anonymous reviewers for their insightful comments.

\appendix


\section{Proofs} \label{app:proofs}

\subsection{Proof of Proposition~\ref{prop:cme}} \label{app:proofpropcme}
\begin{proof}
``if'': Assume that causal minimality is not satisfied. Then, there is an $X_j$ and a $Y \in \PA[\G]{j}$, such that $\lawX$ is also Markov with respect to the graph obtained when removing the edge $Y \rightarrow X_j$ from $\G$.\\
``only if'': If $\lawX$ has a density, the Markov condition is equivalent to the Markov factorization \citep[][Theorem~3.27]{Lauritzen1996}. Assume that $Y \in \PA[\G]{j}$ and $X_j \independent Y\,|\,\PA[\G]{j} \setminus \{Y\}$. Then
$P(\X)= P(X_j|\PA[\G]{j}\setminus\{Y\}) \prod_{k\neq j}P(X_k|\PA[\G]{k})$, which implies that $\lawX$ is Markov w.r.t.\ $\G$ without $Y\rightarrow X_j$.
\end{proof}

\subsection{Proof of Proposition~\ref{prop:tcgunique}\label{app:proof_tcgunique}}
\begin{proof}
We will prove that for all $\G_{1}$ and $\G_{2}$ in $\mathbb{G}$ there is DAG $\G \in \mathbb{G}$ such that $\G \leq \G_1$ and $\G \leq \G_2$. This implies the existence of a least element since the set $\mathbb{G}$ is finite.
Consider any node $X_i$ and denote the $\G_1$-parents by $X_{j_1}, \ldots, X_{j_r}, X_{k_{r+1}}, \ldots, X_{k_{r+s}}$ and the $\G_2$-parents by $X_{j_1}, \ldots, X_{j_r}, X_{\ell_{r+1}}, \ldots, X_{\ell_{r+t}}$, such that $\{k_{r+1}, \ldots, k_{r+s}\}$ and $\{\ell_{r+1}, \ldots, \ell_{r+t}\}$ are disjoint sets.
Here, $X_{j_1}, \ldots, X_{j_r}$ are the joint parents in $\G_1$ and $\G_2$. We have for all $x_{j_1}, \ldots, x_{j_r}$, $x_{k_{r+1}}, \ldots, x_{k_{r+s}}$ and $x_{\ell_{r+1}}, \ldots, x_{\ell_{r+t}}$ (at which the density $p$ is strictly positive) that
\begin{align*}
&p(X_i\,|\,X_{j_1} = x_{j_1}, \ldots, X_{j_r} = x_{j_r}, X_{k_{r+1}} = x_{k_{r+1}}, \ldots, X_{k_{r+s}} = x_{k_{r+s}})\\
= \; &p\big(X_i\,|\, do (X_{j_1} = x_{j_1}, \ldots, X_{j_r} = x_{j_r}, X_{k_{r+1}} = x_{k_{r+1}}, \ldots, X_{k_{r+s}} = x_{k_{r+s}}, \\
&\qquad \qquad \qquad \qquad \qquad \qquad \qquad \quad \; X_{\ell_{r+1}} = x_{\ell_{r+1}}, \ldots, X_{\ell_{r+t}} = x_{\ell_{r+t}})\big)\\
= \; &p(X_i\,|\,X_{j_1} = x_{j_1}, \ldots, X_{j_r} = x_{j_r}, X_{\ell_{r+1}} = x_{\ell_{r+1}}, \ldots, X_{\ell_{r+t}} = x_{\ell_{r+t}}) =: (*)
\end{align*}
This implies
\begin{align*}
(*) = p(X_i\,|\,X_{j_1} = x_{j_1}, \ldots, X_{j_r} = x_{j_r})\,.
\end{align*}
Set the variables $X_{j_1}, \ldots, X_{j_r}$ to be the $\G$-parents of node $X_i$ and repeat for all nodes $X_i$.
The distribution $\lawX$ is Markov w.r.t. graph $\G$ by its construction. 
Note that all proper subgraphs of a true causal DAG with respect to which $\lawX$ is Markov are again true causal DAGs. This proves the statement about causal minimality.
\end{proof}

\subsection{Proof of Proposition~\ref{prop:sur}\label{app:sems}}
\begin{proof}
Let $N_1, \cdots, N_p$ be independent and uniformly distributed between $0$ and $1$. We then define
$
X_j=f_j(\PA[]{j},N_j)
$
with
$$
f_j(x_{\PA[]{j}},n_j) = F^{-1}_{X_j|\PA[]{j}=x_{\PA[]{j}}}(n_j)
$$
where $F^{-1}_{X_j|\PA[]{j}=x_{\PA[]{j}}}$ is the inverse cdf from $X_j$ given $\PA[]{j}=x_{\PA{j}}$.
\end{proof}

\subsection{Proof of Proposition~\ref{prop:nonconst}} \label{app:proofnonconst}
\begin{proof}
Assume causal minimality is not satisfied. We can then find a $j$ and $i \in \PA[]{j}$ with $X_j = f_{j}(X_{\PA{j} \setminus \{i\}},X_i) + N_j$ that does not dependent on $X_i$ if we condition on all other parents $\PA[]{j} \setminus \{i\}$ (Proposition~\ref{prop:cme}). 
Let us denote $\PA[]{j} \setminus \{X_i\}$ by $X_A$. For the function $f_j$ it follows that $f_j(x_A, x_i) = c_{x_A}$ for $\law{X_A, X_i}$-almost all $(x_A, x_i)$. Indeed, assume without loss of generality that $\mean N_j = 0$, take the mean of $X_j \given \PA[\G_0]{j} = (x_A, x_i)$ and use e.g. (2b) from \citep{Dawid79}.
The continuity of $f_{j}$ implies that $f_j$ is constant in its last argument.

The converse statement follows from Proposition~\ref{prop:cme}, too.
\end{proof}

\subsection{Proof of Theorem~\ref{thm:biv}\label{app:proof_biv}}
\begin{proof}%
To simplify notation we write $X:= X_i$ and $Y:= X_j$ (see Definition~\ref{def:ibanm}).
If $\G_0$ is the empty graph, $X \independent Y$. On the other hand, if the graph is not empty, $X \independent Y$ would be a violation of causal minimality.
We can therefore now assume that the graph is not empty and $X \notindependent Y$. Let us assume that the graph is not identifiable and we have
\begin{equation}\label{backward}
p_n(y-f(x))p_x(x) = p(x,y)=p_{\tilde n}(x-g(y))p_{y}(y)\,.
\end{equation}
Set 
\begin{equation}\label{piDef}
\pi(x,y):= \log p(x,y)=\nu (y-f(x))+\xi (x)\,,
\end{equation}
and $\tilde \nu:=\log p_{\tilde n}$,   
$\eta:=\log p_y$.
From the r.h.s.\ of Equation~\eqref{backward} we find 
$
\pi(x,y)= \tilde \nu (x-g(y))+\eta (y)
$,
implying
\[
\frac{\partial^2 \pi}{\partial x\partial y} =-\tilde \nu''(x-g(y)) g'(y) \quad
\hbox{ and } \quad
\frac{\partial^2 \pi}{\partial x^2}=\tilde \nu''(x-g(y))\,.
\]
We conclude 
\begin{equation}\label{wichtig}
\frac{\partial}{\partial x}\left( \frac{\partial^2 \pi/\partial x^2}{
\partial^2 \pi/(\partial x \partial y) }\right)=0\,.
\end{equation}

Using Equation~\eqref{piDef}  we obtain
\begin{equation}\label{partial1}
\frac{\partial^2 \pi}{\partial x\partial y} =-\nu''(y-f(x))f'(x)\,,
\end{equation}
and
\begin{equation}\label{partial2}
 \frac{\partial^2 \pi}{\partial x^2}=\frac{\partial}{\partial x}
\left(-\nu'(y-f(x))f'(x)+\xi'(x)\right)
= \nu'' (f')^2 -\nu'f''+ \xi''\,, 
\end{equation}
where we have dropped the  arguments for convenience.
Combining Equations~\eqref{partial1} and \eqref{partial2} yields
\begin{align*}
\frac{\partial}{\partial x}
\left(\frac{\frac{\partial^2 \pi}{\partial x^2}}{\frac{\partial^2 \pi}{\partial x \partial y}}\right)  
= &-2f'' +\frac{\nu'f'''}{\nu''f'}- \xi'''   \frac{1}{\nu'' f'}+
\frac{\nu'\nu'''f''}{(\nu'')^2}
-\frac{\nu'(f'')^2}{\nu''(f')^2}
-\xi''\frac{\nu'''}{(\nu'')^2}+\xi''\frac{f''}{\nu'' (f')^2}  \,.
\end{align*}

Due to Equation~\eqref{wichtig} this expression  must vanish 
and we obtain the differential equation~\eqref{DGL} 
\begin{equation*}
\xi'''=   \xi''  \left(-\frac{\nu'''f'}{\nu''}
+\frac{f''}{f'}\right) 
-2 \nu''f''f' 
+\nu'f'''+\frac{\nu'\nu'''f''f'}{\nu''}-\frac{\nu'(f'')^2}{f'}\,,
\end{equation*}
by term reordering. This contradicts the assumption that the distribution is generated by an identifiable bivariate additive noise model, see Condition~\ref{cond}.
\end{proof}

\subsection{Proof of Proposition~\ref{prop:biv1}\label{app:proof_biv1}}
\begin{proof}
Let the notation be as in Theorem~\ref{thm:biv} and let $y$ be fixed  such  that  $\nu''(y-f(x))f'(x)\neq 0$ 
holds for all but countably many $x$. 
Given $f,\nu$, we obtain 
a linear inhomogeneous differential equation (DE) for $\xi$:
\begin{equation}\label{DGLmitG}
\xi'''(x)=\xi''(x) G(x,y) +H(x,y)  \,,
\end{equation}
where $G$  and  $H$ are  defined by 
$$
G:= -\frac{\nu'''f'}{\nu''}+\frac{f''}{f'}
$$
and
$$ 
H:= -2\nu''f''f' +\nu'f'''+\frac{\nu'\nu'''f''f'}{\nu''}-\frac{\nu'(f'')^2}{f'} \,,
$$
see proof of Theorem~\ref{thm:biv}. Setting $z:=\xi''$ we have
$
z'(x)=z(x)G(x,y)+H(x,y)\,.
$
Given that such a function  $z$ exists, it is given  by
\begin{equation}\label{rInt}
z(x)=z(x_0) e^{\int_{x_0}^x G(\tilde{x},y) d\tilde{x}}
+\int_{x_0}^x  e^{\int_{\hat{x}}^x G(\tilde{x},y)d\tilde{x}} 
H(\hat{x},y) d\hat{x}\,.
\end{equation}
Then 
$z$ is determined by $z(x_0)$ since we can extend Equation~\eqref{rInt}
 to the remaining points.  The  
set of all functions  $\xi$ satisfying the linear inhomogenous 
DE~\eqref{DGLmitG} is a $3$-dimensional affine space:
Once we have fixed $\xi(x_0),\xi'(x_0),\xi''(x_0)$ 
for some arbitrary point $x_0$, 
$\xi$ is  completely determined. Given fixed $f$ and  $\nu$,  the set of all  
$\xi$ 
admitting a backward model is contained in this subspace.
\end{proof}

\subsection{Proof of Corollary~\ref{cor:gaussianity_implies_linearity}\label{app:proof_gaussianity_implies_linearity}}
\begin{proof}%
Similarly to how \eqref{wichtig} was derived, under the assumption of the existence of a reverse model one can derive
$$\frac{\dd^2 \pi}{\dx\dy} \cdot \frac{\dd}{\dx}\left(\frac{\dd^2 \pi}{\dx^2}\right) = \frac{\dd^2\pi}{\dx^2} \cdot \frac{\dd}{\dx} \left( \frac{\dd^2\pi}{\dx\dy} \right)$$
Now using \eqref{partial1} and \eqref{partial2}, we obtain
\begin{align*}
(-\nu''f') \cdot &\frac{\dd}{\dx}\left(\nu''(f')^2 - \nu'f'' + \xi''\right) 
  = (\nu''(f')^2 - \nu'f'' + \xi'') \cdot \frac{\dd}{\dx}\left(-\nu''f'\right),
\end{align*}
which reduces to
\begin{align*}
-2&(\nu''f')^2f'' + \nu''f'\nu'f''' - \nu''f'\xi''' 
  = -\nu'f''\nu'''(f')^2 + \xi''\nu'''(f')^2 + \nu''\nu'(f'')^2 - \nu''f''\xi''\,.
\end{align*}
Substituting the assumptions $\xi'''=0$ and $\nu'''=0$ (and hence $\nu''=C$ everywhere with $C \ne 0$ since otherwise $\nu$ cannot be a proper log-density) yields
$$\nu'\big(y-f(x)\big) \cdot \big(f'f''' - (f'')^2\big) = 2C(f')^2f'' - f''\xi''.$$
Since $C\neq 0$ there exists an $\alpha$ such that $\nu'(\alpha) = 0$. Then, restricting
ourselves to the submanifold $\{(x,y) \in \RN^2 : y - f(x) = \alpha\}$ on which $\nu' = 0$, we have
$$0 = f'' (2C(f')^2- \xi'').$$
Therefore, for all $x$ in the open set $[f'' \ne 0]$, we have $(f'(x))^2 = \xi'' / (2C)$, which is a constant, so $f'' = 0$ on $[f'' \ne 0]$: a contradiction. Therefore, $f'' = 0$ everywhere.
\end{proof}

\subsection{Definitions of Proposition~\ref{prop:kun}} \label{app:kun}
\begin{definition}{\citep{Zhang2009}}
A one-dimensional distribution that is absolutely continuous with respect to the Lebesgue measure and density $p$ is called:
\begin{itemize}
\item \emph{log-mix-lin-exp} if there are $c_1, c_2, c_3, c_4$ with $c_1 <0$ and $c_2c_3>0$ such that
$$
\log p(x) = c_1 \exp(c_2 x) + c_3 x + c_4
$$
\item \emph{one-sided asymptotically exponential} if there is $c \neq 0$ such that
$$
\frac{d}{dx} \log p(x) \rightarrow c
$$
as $x \rightarrow -\infty$ or $x \rightarrow \infty$.
\item \emph{two-sided asymptotically exponential} if there are $c_1 \neq 0$ and $c_2 \neq 0$ such that
$$
\frac{d}{dx} \log p(x) \rightarrow c_1 
$$
as $x \rightarrow -\infty$ and
$$
\frac{d}{dx} \log p(x) \rightarrow c_2 
$$
as $x \rightarrow \infty$.
\item a \emph{generalized mixture of two exponentials} if there are $d_1, d_2, d_3, d_4, d_5, d_6$
with $d_4 > 0$, $d_3 > 0$, $d_1 d_5 > 0$ and $d_2 < -\frac{d_1}{d_5}$ such that
$$
\log p(x) = d_1 x + d_2 \log(d_3 + d_4 \exp(d_5 x))+ d_6
$$
\end{itemize} 
\end{definition}

\subsection{Proof of Example~\ref{ex:cou2}} \label{app:cou2}

\begin{proof}
Our starting point is the assumption of nonidentifiability. In other words, we can describe the joint distribution
of $x$ and $y$ both as an additive noise model where $X$ causes $Y$, and as an additive noise model where $Y$ causes $X$.
Using the same notation as in Theorem \ref{thm:biv}, this means that:
\begin{equation}\label{eq:nonidentifiable}
  \xi(x) + \nu\big(y - f(x)\big) = \eta(y) + \tnu\big(x - g(y)\big) \qquad \forall x,y \in \RN.
\end{equation}



Case II in Proposition \ref{prop:kun} (reproduced from Table 1 in \cite{Zhang2009}) states that if both $\xi$ and $\nu$ are log-mix-lin-exp and $f$ is affine, then
there could be an unidentifiable model. Let us verify whether that is indeed the case.
We take 
\begin{align*}
  \xi(x) & = c_1 \exp(c_2 x) + c_3 x + c_4 \\
  \nu(n) & = \tc_1 \exp(\tc_2 n) + \tc_3 n + \tc_4 \\
  f(x) & = ax + b
\end{align*}
with $a \ne 0$ ($a = 0$ is the degenerate case with $X$ and $Y$ independent).

We can rewrite (\ref{eq:nonidentifiable}) as follows, by substituting $x$ with $x + g(y)$:
\begin{equation}\label{eq:case2}
  c_1 e^{c_2 (x + g(y))} + c_3 (x + g(y)) + c_4 + \tc_1 e^{\tc_2 (y - ax - ag(y) - b)} + \tc_3 (y - ax - ag(y) - b) + \tc_4 = \eta(y) + \tnu(x).
\end{equation}
Differentiating with respect to $x$:
\begin{equation}\label{eq:case2_dx}
  c_1 c_2 e^{c_2 (x + g(y))} + c_3 - a \tc_1 \tc_2 e^{\tc_2 (y - ax - ag(y) - b)} - \tc_3 a = \tnu'(x).
\end{equation}
Differentiating with respect to $y$:
$$c_1 c_2^2 e^{c_2 (x + g(y))} g'(y) - a \tc_1 \tc_2^2 e^{\tc_2 (y - ax - ag(y) - b)} (1 - ag'(y)) = 0.$$
This can only be satisfied for all $x$ if $c_2 = -a \tc_2$. In that case:
$$-a c_1 g'(y) + \tc_1 e^{\tc_2 (y - b)} (1 - ag'(y)) = 0.$$
Rewriting:
$$a g'(y) = \frac{\tc_1 e^{\tc_2 (y - b)}}{c_1 + \tc_1 e^{\tc_2 (y - b)}}.$$
Integrating:
$$g(y) = -\frac{1}{c_2} \ln (-c_1 - \tc_1 e^{\tc_2(y-b)}) + \frac{C}{c_2}.$$
Note that:
$$e^{c_2 g(y)} = -\frac{1}{c_1 + \tc_1 e^{\tc_2(y-b)}} e^{-C}.$$
Substituting into (\ref{eq:case2_dx}):
$$-c_2 e^{-C} e^{c_2 x} + c_3 - \tc_3 a = \tnu'(x).$$
Integrating:
$$-e^{-C} e^{c_2 x} + (c_3 - \tc_3 a) x + \td_4 = \tnu(x).$$
which is also log-mix-lin-exp with parameters $\td_1 = -e^{-C}$, $\td_2 = c_2$, $\td_3 = c_3 - \tc_3 a$, $\td_4$.
Substituting into (\ref{eq:case2}):
$$g(y)(c_3 - \tc_3 a) + \tc_3 y + c_4 - \tc_3 b + \tc_4 - \td_4 = \eta(y),$$
i.e.:
$$\eta(y) = \left(-\frac{1}{c_2} \ln (-c_1 - \tc_1 e^{\tc_2(y-b)}) + \frac{C}{c_2}\right) (c_3 - \tc_3 a) + \tc_3 y + c_4 - \tc_3 b + \tc_4 - \td_4.$$
This gives an inequality constraint: $c_3 \ne a \tc_3$. $\eta(y)$ is
a generalized mixture of exponentials distribution with parameters $d_1 = \tc_3$, $d_2 = -\frac{c_3 - a\tc_3}{c_2}$, $d_3 = -c_1$, $d_4 = -\tc_1 e^{-\tc_2 b}$, $d_5 = \tc_2$, $d_6 = C \frac{c_3 - a\tc_3}{c_2} + c_4 - \tc_3b + \tc_4 - \td_4$. One can check that all constraints on the parameters of the generalized mixture of exponentials are satisfied. Choosing $C$ appropriately allows for normalizing the log-density.
One can also easily verify that with these choices of $\tnu(x)$ and $\eta(y)$, equation \eref{eq:nonidentifiable} holds, and therefore this gives an example of a nonidentifiable additive noise model.
\end{proof}

\subsection{Some Lemmata} \label{sec:lem}
The following four statements are all plausible and their proof is mostly about technicalities. The reader may skip to the next section 
and use the lemmata whenever needed. For random variables $A$ and $B$ we use $A \given_{B=b}$ to denote the random variable $A$ after conditioning on $B=b$ (assuming densities exist and $B$ has positive density at $b$).

\begin{lemma}\label{lem:cond}
Let $Y \in \C{Y},N \in \C{N},\B{Q} \in \C{Q},\B{R} \in \C{R}$ be random variables whose joint distribution is absolutely continuous with respect
to some product measure ($\B{Q}$ and $\B{R}$ can be multivariate) and with density $p_{Y,\B{Q},\B{R}, N}(y,\B{q},\B{r},n)$. Let $f : \C{Y} \times \C{Q} \times \C{N} \to \R$ be a measurable function.
If $N \independent (Y,\B{Q},\B{R})$ then for all $\B{q} \in \C{Q},\B{r} \in \C{R}$ with $p_{\B{Q},\B{R}}(\B{q},\B{r}) > 0$:
$$
f(Y,\B{Q},N){\given}_{\B{Q}=\B{q},\B{R}=\B{r}} \overset{\mathcal{L}}{=} f(Y{\given}_{\B{Q}=\B{q},\B{R}=\B{r}},\B{q},N)\,.
$$
\end{lemma}
A formal proof of this statement can be found in \citep[][Lemma~2]{Peters2011b}.

\begin{lemma}\label{lem:noi}
Let $\lawX$ be generated according to a \SEM as in~\eqref{eq:sem} with corresponding \DAG $\G$ and consider a variable $X \in \B{X}$. If $\B{S} \subseteq \ND[\G]{X}$ then $N_X \independent \B{S}$.
\end{lemma}
\begin{proof}
Write $\B{S}=\{S_1, \ldots, S_k\}$. Then
$$
\B{S}=\big(f_{S_1}(\PA[\G]{S_1}, N_{S_1}), \ldots, f_{S_k}(\PA[\G]{S_k}, N_{S_k})\big)\,.
$$
Again, one can substitute the parents of $S_i$ by the corresponding functional equations and proceed recursively. After finitely many steps one obtains
$\B{S}=f(N_{T_1}, \ldots, N_{T_l})$,
where $\{T_1, \ldots, T_l\}$ is the set of \emph{all} ancestors of nodes in $\B{S}$, which does not contain $X$. Since all noise variables are jointly independent we have $N_X \independent \B{S}$.
\end{proof}

With the intersection property of conditional independence \citep[e.g., 1.1.5 in ][]{Pearl2009}, Proposition~\ref{prop:cme} has the following corollary that we formalize as a lemma.
\begin{lemma} \label{lem:cmc}
Consider the random vector $\X$ and assume that the joint distribution has a (strictly) positive density. Then 
$\lawX$ satisfies causal minimality with respect to $\G$ if and only if $\forall B \in \B{X}$ $\forall A \in \PA[\G]{B} \,$ and $\forall \B{S} \subset \X$ with  $\PA[\G]{B} \setminus \{A\} \subseteq \B{S} \subseteq \ND[\G]{B} \setminus \{A\}$ we have that 
$$
B \notindependent A\,\mid\,\B{S}\,.
$$
\end{lemma}
\begin{proof}
The ``if'' part is immediate. For the ``only if'' let us 
denote $\B{P}:= \PA[\G]{B} \setminus \{A\}$ and $\B{Q}:= \B{S} \setminus (\PA[\G]{B} \setminus \{A\})$, such that $\B{S} = \B{P} \cup \B{Q}$. Observe that $B \notindependent A \given \B{P}$ (see Proposition~\ref{prop:cme}) implies $B \notindependent (\{A\} \cup \B{Q}) \given \B{P}$. From the Markov condition we have $B \independent \B{Q} \given (\B{P} \cup \{A\})$. The intersection property of conditional independence yields $B \notindependent A \given (\B{P} \cup \B{Q})$.
\end{proof}


\subsection{Proof of Theorem~\ref{thm:mul}\label{app:proof_mul}}
\begin{proof}
We assume that there are two restricted additive noise models (see Definition~\ref{def:wnn}) that both induce $\law{\B{X}}$
, one with graph $\G$, the other with graph $\Gp$.
We will show that $\G = \Gp$.
Consider the variables $L, Y$ from Proposition~\ref{prop:graph} (i) and define the sets $\B{Q}:=\PA[\G]{L} \setminus \{Y\}$, $\B{R}:= \PA[\Gp]{Y} \setminus \{L\}$ and $\B{S}:=\B{Q} \cup \B{R}$.
At first, we consider any $\B{s}=(\B{q}, \B{r})$ and write
$
L^* := L \given_{\B{S} = \B{s}}$ and $Y^* := Y \given_{ \B{S} = \B{s}}
$. 
Lemma \ref{lem:noi} gives us $N_L \independent (Y,\B{S})$ and $N_Y \independent (L,\B{S})$ and we can thus apply Lemma \ref{lem:cond}. 
From $\G$ we find
$$
L^* = f_L(\B{q}, Y^*) + N_L, \qquad N_L \independent Y^*
$$
and from $\Gp$ we have
$$
Y^* = g_{Y}(\B{r}, L^*) +  N_{Y}, \qquad N_{Y} \independent L^*
$$
This contradicts Theorem~\ref{thm:biv} since according to Definition~\ref{def:wnn} we can choose $\B{s}=(\B{q}, \B{r})$ 
such that 
$(f_L(\B{q}, \cdot), \law{Y^*}, \law{N_L})$ and 
$(g_Y(\B{r}, \cdot), \law{L^*}, \law{N_Y})$ satisfy Condition~\ref{cond}.

\end{proof}

\subsection{Proof of Proposition~\ref{prop:graph}} \label{app:propmain}
\begin{proof}
Since \DAGs do not contain any cycles, we always find nodes that have no descendants (start a directed path at some node: after at most $\#\B{X}-1$ steps we reach a node without a child). Eliminating such a node from the graph leads to a DAG, again; we can discard further nodes without children in the new graph. We repeat this process for all nodes that have no children in both $\G$ and $\Gp$ and have the same parents in both graphs. If we end up with no nodes left, the two graphs are identical which violates the assumption of the proposition. Otherwise, we end up with a smaller set of variables that we again call $\X$, two smaller graphs that we again call $\G$ and $\Gp$ and a node $L$ that has no children in $\G$ and either $\PA[\G]{L}\neq \PA[\Gp]{L}$ or $\CH[\Gp]{L} \neq \emptyset$.
We will show that this leads to a contradiction. Importantly, because of the Markov property of the distribution with respect to $\G$, all other nodes are independent of $L$ given $\PA[\G]{L}$:
\begin{equation} \label{eq_xind}
L \independent \B{X} \setminus (\PA[\G]{L} \cup \{L\}) \,\given\, \PA[\G]{L}\,.
\end{equation}

To make the arguments easier to understand, we introduce the following notation (see also Fig.~\ref{fig:wcp}): we partition $\G$-parents of $L$ into $\B{Y}, \B{Z}$ and $\B{W}$. Here, $\B{Z}$ are also $\Gp$-parents of $L$, $\B{Y}$ are $\Gp$-children of $L$ and $\B{W}$ are not adjacent to $L$ in $\Gp$. We denote with $\B{D}$ the $\Gp$-parents of $L$ that are not adjacent to $L$ in $\G$ and by $\B{E}$ the $\Gp$-children of $L$ that are not adjacent to $L$ in $\G$. 
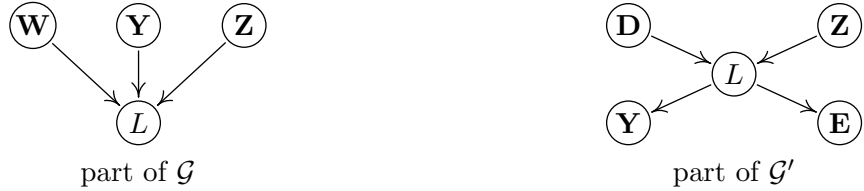
\begin{figure}[ht]
  \begin{minipage}[t]{0.48\columnwidth}
    \begin{center}
      \begin{tikzpicture}[scale=1.08, line width=0.5pt, minimum size=0.58cm, inner sep=0.3mm, shorten >=1pt, shorten <=1pt]
	\normalsize
        \draw (0,0) node(x) [circle, draw] {$L$};
        \draw (-1.3,1.2) node(w) [circle, draw] {$\B{W}$};
        \draw (0,1.2) node(y) [circle, draw] {$\B{Y}$};
        \draw (1.3,1.2) node(z) [circle, draw] {$\B{Z}$};
        \draw[-arcsq] (z) -- (x);
        \draw[-arcsq] (y) -- (x);
        \draw[-arcsq] (w) -- (x);
      \end{tikzpicture}\\
      part of $\G$
    \end{center}
  \end{minipage}
  \hspace{0.02\columnwidth}
  \begin{minipage}[t]{0.48\columnwidth}
    \begin{center}
      \begin{tikzpicture}[scale=1.08, line width=0.5pt, minimum size=0.58cm, inner sep=0.3mm, shorten >=1pt, shorten <=1pt]
	\normalsize
        \draw (0,-0.6) node(x) [circle, draw] {$L$};
        \draw (-1.3,0) node(d) [circle, draw] {$\B{D}$};
        \draw (1.3,0) node(z) [circle, draw] {$\B{Z}$};
        \draw (1.3,-1.2) node(e) [circle, draw] {$\B{E}$};
        \draw (-1.3,-1.2) node(y) [circle, draw] {$\B{Y}$};
        \draw[-arcsq] (z) -- (x);
        \draw[-arcsq] (d) -- (x);
        \draw[-arcsq] (x) -- (y);
        \draw[-arcsq] (x) -- (e);
      \end{tikzpicture}\\
      part of $\Gp$
    \end{center}
  \end{minipage}
  \caption{Nodes adjacent to $L$ in $\G$ and $\Gp$ \label{fig:wcp}}
\end{figure}
Thus:
$\PA[\G]{L} = \B{Y} \cup \B{Z} \cup \B{W}$, $\CH[\G]{L} = \emptyset,$
$\PA[\Gp]{L} = \B{Z} \cup \B{D}$, $\CH[\Gp]{L} = \B{Y} \cup \B{E}$.
Consider $\B{T} := \B{W} \cup \B{Y}$. We distinguish two cases:\\

\noindent
Case (i):
$\B{T} = \emptyset$.\\
Then there must be a node $D \in \B{D}$ or a node $E \in \B{E}$, otherwise $L$ would have been discarded.
\begin{itemize}
\item[1.] If there is a $D \in \B{D}$ then \eref{eq_xind} implies
$L \independent D \given \B{S}$
for $\B{S} := \B{Z} \cup \B{D} \setminus \{D\}$, 
which contradicts Lemma~\ref{lem:cmc} (applied to $\Gp$).
\item[2.] 
If $\B{D}=\emptyset$ and there is $E \in \B{E}$ then 
$E \independent L \given \B{S}$
holds for $\B{S}:=\B{Z} \cup \PA[\Gp]{E} \setminus \{L\}$ (see graph $\G$), which also contradicts Lemma \ref{lem:cmc} (note that $\B{Z}\subseteq \ND[\Gp]{E}$ to avoid cycles).
\end{itemize}
Case (ii):
$\B{T} \ne \emptyset$.\\
Then $\B{T}$ contains a ``$\Gp$-youngest'' node with the property that there is no
directed $\Gp$-path from this node to any other node in $\B{T}$. 
This node may not be unique.
\begin{itemize}
  \item[1.] Suppose that some $W \in \B{W}$ is such a youngest node. Consider the \DAG $\tilde \Gp$ that equals $\Gp$ with additional edges $Y \rightarrow W$ and $W' \rightarrow W$ for all $Y \in \B{Y}$ and $W' \in \B{W} \setminus \{W\}$. In $\tilde \Gp$, $L$ and $W$ are not adjacent. Thus we find a set $\B{\tilde S}$ such that $\B{\tilde S}$ $d$-separates $L$ and $W$ in $\tilde \Gp$; indeed, one can take 
$\B{\tilde S} = \PA[\tilde \Gp]{L}$ if $W \notin \DE[\tilde \Gp]{L}$ and $\B{\tilde S} := \PA[\tilde \Gp]{W}$ if $L \notin \DE[\tilde \Gp]{W}$.
Then also $\B{S}=\B{\tilde S} \cup \{\B{Y},\B{Z},\B{W}\setminus \{W\}\}$ $d$-separates $L$ and $W$ in $\tilde \Gp$.
\par
\begingroup
\leftskip=0.4cm 
\noindent 
Indeed, all $Y \in \B{Y}$ are already in $\B{\tilde S}$ in order to block $L \rightarrow Y \rightarrow W$. Suppose there is a $\tilde \G'$-path that is blocked by $\B{\tilde S}$ and unblocked if we add $Z$ and $W'$ nodes to $\B{\tilde S}$. 
How can we unblock a path by including more nodes? The path ($L \cdots V_1 \cdots U_1 \cdots W$ in Fig.~\ref{fig:indeed}) must contain a collider $V_1$ that is an ancestor of a $Z$ with $V_1, \ldots, V_m, Z \notin \B{\tilde S}$ and corresponding nodes $U_i$ for a $W'$ node. 
Choose $V_1$ and $U_1$ on the given path so close to each other such that there is no such collider in between. If there is no $V_1$, choose $U_1$ closest to $L$, if there is no $U_1$, choose $V_1$ closest to $W$. Now the path $L \leftarrow Z \cdots V_1 \cdots U_1 \cdots W' \rightarrow W$ is unblocked given $\B{\tilde S}$, which is a contradiction to the assumption that $\B{\tilde S}$ $d$-separates $L$ and $W$.
\par
\endgroup
But then $\B{S}$ $d$-separates $L$ and $W$ in $\Gp$, too (there are less paths), and we have
$
L \independent W \,\mid\,\B{S},
$
which contradicts Lemma \ref{lem:cmc} (applied to $\G$).
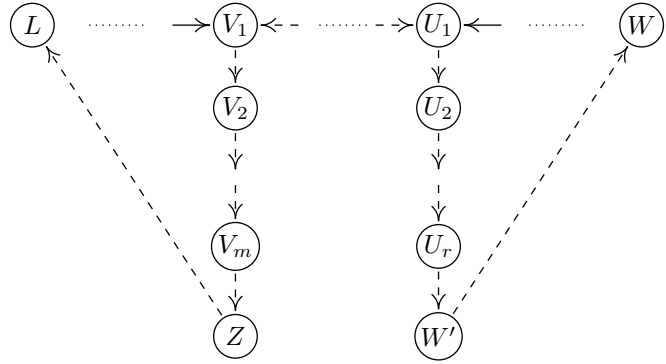
\begin{figure}[ht]
    \begin{center}
      \begin{tikzpicture}[xscale=0.9, yscale=0.92, line width=0.5pt, minimum size=0.58cm, inner sep=0.3mm, shorten >=1pt, shorten <=1pt]
	\small
        \draw (-3,0) node(x) [circle, draw] {$L$};
        \draw (6,0) node(w) [circle, draw] {$W$};
        \draw (0,0) node(v1) [circle, draw] {$V_1$};
        \draw (0,-1.2) node(v2) [circle, draw] {$V_2$};
        \draw (0,-3.2) node(vm) [circle, draw] {$V_m$};
        \draw (0,-4.5) node(z) [circle, draw] {$Z$};
        \draw (3,0) node(u1) [circle, draw] {$U_1$};
        \draw (3,-1.2) node(u2) [circle, draw] {$U_2$};
        \draw (3,-3.2) node(us) [circle, draw] {$U_r$};
        \draw (3,-4.5) node(wp) [circle, draw] {$W'$};
        \draw (-1.3,0) node(x3) [circle, draw, white] {$L$};
        \draw (1.3,0) node(x4) [circle, draw, white] {$L$};
        \draw (1.7,0) node(x5) [circle, draw, white] {$L$};
        \draw (4.3,0) node(x6) [circle, draw, white] {$L$};
        \draw (0,-2.45) node(v3) [circle, draw, white] {$L$};
        \draw (0,-1.95) node(v4) [circle, draw, white] {$L$};
        \draw (3,-2.45) node(u3) [circle, draw, white] {$L$};
        \draw (3,-1.95) node(u4) [circle, draw, white] {$L$};
        \draw (1.7,0) node(x5) [circle, draw, white] {$L$};
        \draw (4.3,0) node(x6) [circle, draw, white] {$L$};
        \draw[-arcsq, dashed] (z) -- (x);
        \draw[-arcsq] (x3) -- (v1);
        \draw[-arcsq, dashed] (x4) -- (v1);
        \draw[-arcsq, dashed] (wp) -- (w);
        \draw[-arcsq, dashed] (v1) -- (v2);
        \draw[-arcsq, dashed] (v2) -- (v3);
        \draw[-arcsq, dashed] (v4) -- (vm);
        \draw[-arcsq, dashed] (vm) -- (z);
        \draw[-arcsq, dashed] (x5) -- (u1);
        \draw[-arcsq] (x6) -- (u1);
        \draw[-arcsq, dashed] (u1) -- (u2);
        \draw[-arcsq, dashed] (u2) -- (u3);
        \draw[-arcsq, dashed] (u4) -- (us);
        \draw[-arcsq, dashed] (us) -- (wp);
	\draw[dotted] (-2.2,0) -- (-1.3,0);
	\draw[dotted] (1.2,0) -- (2.0,0);
	\draw[dotted] (4.3,0) -- (5.2,0);
      \end{tikzpicture}
    \end{center}
  \caption{Assume the path $L \cdots V_1 \cdots U_1 \cdots W$ is blocked by $\B{\tilde S}$, but unblocked if we include $Z$ and $W'$. Then the dashed path is unblocked given $\B{\tilde S}$. \label{fig:indeed}}
\end{figure}
\item[2.] Therefore, the $\Gp$-youngest node in $\B{T}$ must be some $Y \in \B{Y}$.\\
Define $\B{Q}:=\PA[\G]{L} \setminus \{Y\}$, $\B{R}:= \PA[\Gp]{Y} \setminus \{L\}$ and
$\B{S}:=\B{Q} \cup \B{R}$. 
Clearly, $\B{S} \subseteq \ND[\G]{L} \setminus \{Y\}$ since $L$ does not have any descendants in $\G$. Further, $\B{S} \subseteq \ND[\Gp]{Y} \setminus \{L\}$ because $Y$ is the $\Gp$-youngest under all $\B{W}$ and $\B{Y}\setminus\{Y\}$ by construction and any directed path from $Y$ to $Z \in \B{Z}$ would introduce a cycle in $\Gp$. Ergo, $\{Y\} \cup \B{S} \subseteq \ND[\G]{L}$ and $\{L\} \cup \B{S} \subseteq \ND[\Gp]{Y}$. 
\end{itemize}
The variables $L$ and $Y$ and the sets $\B{Q}, \B{R}$ and $\B{S}$ satisfy the conditions required in statement (i) of Proposition~\ref{prop:graph}. 

Statement (ii) follows as a special case since for Markov equivalent graphs, $\B{W}, \B{D}$ and $\B{E}$ are all empty. Consider the $\Gp$-youngest node $Y$. In order to avoid $v$-structures appearing in $\G$ and not in $\Gp$ all nodes $Z \in \B{Z}$ are directly connected to the $\Gp$-youngest $Y$. And to avoid cycles, those nodes $Z \in \B{Z}$ are $\Gp$-parents of $Y$.
The node $Y$ cannot have other parents except for the ones in $\B{Y}$ and $\B{Z}$ since this would introduce $v$-structures in $\Gp$ (with collider $Y$) that do not appear in $\G$.
\end{proof}

\subsection{Proof of Corollary~\ref{cor:new}} \label{app:proofcornew}
\begin{proof}
We only prove (i) since (ii) is a special case.
Causal minimality is satisfied because of Proposition~\ref{prop:nonconst}.
We can then apply exactly the same proof as in Theorem~\ref{thm:mul}. This yields the two equations
\begin{align*}
L^* &= f_{L}(\B{q},Y^*) + N_L, \qquad N_L \independent Y^* \qquad \text{and}\\
Y^* &= g_{Y}(\B{r},L^*) +  N_{Y}, \qquad N_{Y} \independent L^*
\end{align*}
Since $N_L$ is Gaussian, Proposition~\ref{prop:kun} would imply that $f_L(\B{q},\cdot)$ is linear. This contradicts the assumption of nonlinearity.
It therefore remains to show that Proposition~\ref{prop:kun} is applicable. Let us define $f := f_L(\B{q},\cdot)$ and $g := g_Y(\B{r},\cdot)$ and suppose that $f'(n_y)=0$. As in the proof of Theorem~1 in \citep[][below equation~(7)]{Zhang2009}, we can conclude that for all $\ell^*$
$$
-\frac{1}{\sigma_Y^2} g'(\ell^*) + \frac{\ell^*}{\sigma_L^2} f''(n_y) g'(\ell^*) = 0
$$
which implies $g'(\ell^*)=0$ for all $\ell^*$. This contradicts the nonlinearity assumption of $g$.
\end{proof}

\subsection{Proof of Lemma~\ref{lem:nocm}} \label{app:prooftoporder}
\begin{proof}
For (a) we can change the corresponding structural equation $X_j = f_j(\PA[\G_0]{j}) + N_j$ into $X_j = \tilde f_j(\PA[\G_0]{j}, X_i) + N_j$ where $\tilde f_j$ equals $f_j$ in the first $\# \PA[\G_0]{j}$ components and $\tilde f_j$ is constant in the last component. 

We now prove statement (b). By Proposition~\ref{prop:nonconst}, $\G_0$ contains an edge $i \rightarrow j$ such that 
(with $X_A := \PA[]{j} \setminus \{X_i\}$)
the function $f_j$ in $X_j = f_j(X_A, X_i) + N_j$ is constant in its argument of $X_i$, that is $f_j(x_A, x_i) = c_{x_A}$ for all $x_i$.
We can then construct a new additive noise model by defining $X_j = \tilde f_j(X_A) + N_j$ where $\tilde f_j(x_A) = c_{x_A}$ for all $x_A$ and keeping all other equations as well as the noise variables the same. By the Markov factorization we obtain the same joint distribution. Iterating this procedure proves the lemma: hereby, the intersection property \citep[1.1.5 in ][]{Pearl2009} ensures that if two edges $i_1 \rightarrow j$ and $i_2 \rightarrow j$ can be removed one after each other, the order does not matter (which is not necessarily true for densities that are not strictly positive). The last argument proves uniqueness of $\G_0^{min}$.
\end{proof}

\subsection{Proof of Corollary~\ref{cor:toporder}} \label{app:prooftoporder2}
\begin{proof}
For each permutation $\pi$ we only consider ANMs for which the ANM constructed according to the minimal graph $\G_{\pi}^{\text{full, min}}$ (Lemma~\ref{lem:nocm},b) are restricted ANMs.
According to Lemma~\ref{lem:nocm} (a), we can find such ANMs for all graphs $\G_{\pi}^{\text{full}}$ with $\pi \in \Pi^0$. If $\pi \notin \Pi^0$,  Theorem~\ref{thm:mul} implies that $\G_{\pi}^{\text{full, min}} = \G_{0}^{\text{min}}$ which contradicts $\pi \notin \Pi^0$.
\end{proof}

\subsection{Proof of Theorem~\ref{thm:algo}} \label{app:proofalgo}
\begin{proof}
For the correct graph, we know that $N_i$ is independent of all ancestor variables $X_j$ since the latter can be expressed in terms of noise variables without $N_i$. The correct sink nodes therefore lead to independence in step $7$ of Algorithm~\ref{alg:icml}. We will now show that ``wrong sinks'', that is nodes who are not sinks in the correct graph $\G_0$ do not lead to independent residuals in the first iteration of Phase 1. It follows by induction that this is true for any later iteration, too.
Suppose that node $Y$ is not a sink in $\G_0$ but leads to independent residuals (step $7$). Since $Y$ is not a sink in $\G_0$, $Y$ has children in $\G_0$. Call $Z$ the $\G_0$-youngest child, that is there is no directed path from $Z$ to any other child of $Y$. Disregard all descendants of $Z$ and denote the remaining set of variables $\B{S}:= \B{X} \setminus \{Y,Z,\DE[\G_0]{Z}\}$. It therefore follows that 
\begin{equation} \label{eq:helppp}
\DE[\G_0]{Z} \independent Y \given \B{S} \cup \{Z\}\,.
\end{equation} 
Because $Y$ leads to independent residuals we can think of a graph $\G$ in which all variables are parents of $Y$. From Equation~\eqref{eq:helppp} it follows that
$Y = g_Y(\B{S}, Z) + \tilde N_Y$ with $\tilde N_Y \independent (\B{S},Z)$. 
We then proceed similarly as in the proof of Theorem~\ref{thm:mul} and find from $\G_0$ that
$$
Z\given_{\B{S} = \B{s}} = f_Z(\B{s}, Y\given_{\B{S} = \B{s}}) + N_Z\,.
$$
From $\G$ we conclude that
$$
Y\given_{\B{S} = \B{s}} = g_Y(\B{s}, Z\given_{\B{S} = \B{s}}) + \tilde N_Y\,.
$$
Again, this contradicts Theorem~\ref{thm:biv}. 
The correctness of Phase 2 follows from causal minimality and Lemma~\ref{lem:cmc}.
\end{proof}



\bibliography{bibliography}

\begin{thebibliography}{67}
\providecommand{\natexlab}[1]{#1}
\providecommand{\url}[1]{\texttt{#1}}
\expandafter\ifx\csname urlstyle\endcsname\relax
  \providecommand{\doi}[1]{doi: #1}\else
  \providecommand{\doi}{doi: \begingroup \urlstyle{rm}\Url}\fi

\bibitem[Acid and de~Campos(2003)]{Acid2003}
S.~Acid and L.~M. de~Campos.
\newblock Searching for {B}ayesian network structures in the space of
  restricted acyclic partially directed graphs.
\newblock \emph{Journal of Artificial Intelligence Research}, 18:\penalty0
  445--490, 2003.

\bibitem[Bergsma(2004)]{Bergsma2004}
W.~P. Bergsma.
\newblock \emph{Testing conditional independence for continuous random
  variables}, 2004.
\newblock EURANDOM-report 2004-049.

\bibitem[Bollen(1989)]{Bollen1989}
K.~A. Bollen.
\newblock \emph{Structural Equations with Latent Variables}.
\newblock John Wiley \& Sons, 1989.

\bibitem[{B{\"u}hlmann} et~al.(2013){B{\"u}hlmann}, Peters, and
  Ernest]{BuhlmannPetersErnest2013}
P.~{B{\"u}hlmann}, J.~Peters, and J.~Ernest.
\newblock {CAM}: Causal additive models, high-dimensional order search and
  penalized regression.
\newblock \emph{ArXiv e-prints (1207.5136)}, 2013.

\bibitem[Chickering(1995)]{Chickering1995}
D.~M. Chickering.
\newblock A transformational characterization of equivalent {B}ayesian network
  structures.
\newblock In \emph{Proceedings of the 11th Annual Conference on {U}ncertainty
  in {A}rtificial {I}ntelligence ({UAI})}, 1995.

\bibitem[Chickering(1996)]{Chickering1996}
D.~M. Chickering.
\newblock Learning {B}ayesian networks is {NP}-complete.
\newblock In \emph{Learning from Data: Artificial Intelligence and Statistics
  V}. Springer-Verlag, 1996.

\bibitem[Chickering(2002)]{Chickering2002}
D.~M. Chickering.
\newblock Optimal structure identification with greedy search.
\newblock \emph{Journal of Machine Learning Research}, 3:\penalty0 507--554,
  2002.

\bibitem[Comon(1994)]{ica_paper}
P.~Comon.
\newblock Independent component analysis -- a new concept?
\newblock \emph{Signal Processing}, 36:\penalty0 287--314, 1994.

\bibitem[Darmois(1953)]{darmois}
G.~Darmois.
\newblock Analyse g\'{e}n\'{e}rale des liaisons stochastiques.
\newblock \emph{Revue de l'Institut International de Statistique}, 21:\penalty0
  2--8, 1953.

\bibitem[Dawid(1979)]{Dawid79}
A.~P. Dawid.
\newblock Conditional independence in statistical theory.
\newblock \emph{Journal of the Royal Statistical Society. Series B},
  41\penalty0 (1):\penalty0 1--31, 1979.

\bibitem[{D}eutscher {W}etterdienst(2008)]{DWD}
{D}eutscher {W}etterdienst.
\newblock Climate data.
\newblock \url{http://www.dwd.de/}, 2008.

\bibitem[Druzdzel and van Leijen(2001)]{Druzdzel2001}
M.~J. Druzdzel and H.~van Leijen.
\newblock Causal reversibility in {B}ayesian networks.
\newblock \emph{Journal of Experimental and Theoretical Artificial
  Intelligence}, 13\penalty0 (1):\penalty0 45--62, 2001.

\bibitem[Eberhardt and Scheines(2007)]{Eberhardt2007}
F.~Eberhardt and R.~Scheines.
\newblock Interventions and causal inference.
\newblock \emph{Philosophy of Science}, 74\penalty0 (5):\penalty0 981--995,
  2007.

\bibitem[Friedman and Koller(2003)]{FriedmanKoller:MLJ03}
N.~Friedman and D.~Koller.
\newblock Being {Bayesian} about {Bayesian} network structure: {A} {Bayesian}
  approach to structure discovery in {Bayesian} networks.
\newblock \emph{Machine Learning}, 50:\penalty0 95--125, 2003.

\bibitem[Fukumizu et~al.(2008)Fukumizu, Gretton, Sun, and
  Sch\"olkopf]{Fukumizu2008}
K.~Fukumizu, A.~Gretton, X.~Sun, and B.~Sch\"olkopf.
\newblock Kernel measures of conditional dependence.
\newblock In \emph{{A}dvances in {N}eural {I}nformation {P}rocessing {S}ystems
  20 ({NIPS})}, 2008.

\bibitem[Geiger and Heckerman(1994)]{Geiger1994}
D.~Geiger and D.~Heckerman.
\newblock Learning {G}aussian networks.
\newblock In \emph{Proceedings of the 10th Annual Conference on {U}ncertainty
  in {A}rtificial {I}ntelligence ({UAI})}, 1994.

\bibitem[Gretton et~al.(2005)Gretton, Herbrich, Smola, Bousquet, and
  Sch\"olkopf]{Gretton2005JMLR}
A.~Gretton, R.~Herbrich, A.~Smola, O.~Bousquet, and B.~Sch\"olkopf.
\newblock Kernel methods for measuring independence.
\newblock \emph{Journal of Machine Learning Research}, 6:\penalty0 2075--2129,
  2005.

\bibitem[Gretton et~al.(2008)Gretton, Fukumizu, Teo, Song, Sch{\"o}lkopf, and
  Smola]{Gretton2008}
A.~Gretton, K.~Fukumizu, C.~H. Teo, L.~Song, B.~Sch{\"o}lkopf, and A.~Smola.
\newblock A kernel statistical test of independence.
\newblock In \emph{{A}dvances in {N}eural {I}nformation {P}rocessing {S}ystems
  20 ({NIPS})}, 2008.

\bibitem[Haughton(1988)]{Haughton1988}
D.~M.~A. Haughton.
\newblock On the choice of a model to fit data from an exponential family.
\newblock \emph{The Annals of Statistics}, 16:\penalty0 342--355, 1988.

\bibitem[Heckerman(1997)]{Heckerman1997}
D.~Heckerman.
\newblock A {B}ayesian approach to causal discovery.
\newblock Technical report, Microsoft Research (MSR-TR-97-05), 1997.

\bibitem[Heckerman and Geiger(1995)]{Heckerman1995}
D.~Heckerman and D.~Geiger.
\newblock Likelihoods and parameter priors for {B}ayesian networks.
\newblock Technical report, Microsoft Research (MSR-TR-95-54), 1995.

\bibitem[Hoyer et~al.(2008)Hoyer, Shimizu, Kerminen, and
  Palviainen]{Hoyer2008b}
P.~Hoyer, S.~Shimizu, A.~J. Kerminen, and M.~Palviainen.
\newblock Estimation of causal effects using linear non-{G}aussian causal
  models with hidden variables.
\newblock \emph{International Journal of Approximate Reasoning}, 49:\penalty0
  362--378, 2008.

\bibitem[Hoyer et~al.(2009)Hoyer, Janzing, Mooij, Peters, and
  Sch\"olkopf]{Hoyer2008}
P.~Hoyer, D.~Janzing, J.~M. Mooij, J.~Peters, and B.~Sch\"olkopf.
\newblock Nonlinear causal discovery with additive noise models.
\newblock In \emph{{A}dvances in {N}eural {I}nformation {P}rocessing {S}ystems
  21 ({NIPS})}, 2009.

\bibitem[Hyv\"arinen and Smith(2013)]{Hyvarinen2013}
A.~Hyv\"arinen and S.~M. Smith.
\newblock Pairwise likelihood ratios for estimation of non-{G}aussian
  structural equation models.
\newblock \emph{Journal of Machine Learning Research}, 14:\penalty0 111--152,
  2013.

\bibitem[Janzing and Sch\"olkopf(2010)]{Janzing2010a}
D.~Janzing and B.~Sch\"olkopf.
\newblock Causal inference using the algorithmic {M}arkov condition.
\newblock \emph{IEEE Transactions on Information Theory}, 56:\penalty0
  5168--5194, 2010.

\bibitem[Janzing and Steudel(2010)]{Steudel2010}
D.~Janzing and B.~Steudel.
\newblock Justifying additive-noise-model based causal discovery via
  algorithmic information theory.
\newblock \emph{Open Systems and Information Dynamics}, 17:\penalty0 189--212,
  2010.

\bibitem[Janzing et~al.(2009)Janzing, Peters, Mooij, and
  Sch\"{o}lkopf]{Janzing2009}
D.~Janzing, J.~Peters, J.~M. Mooij, and B.~Sch\"{o}lkopf.
\newblock Identifying confounders using additive noise models.
\newblock In \emph{Proceedings of the 25th Annual Conference on {U}ncertainty
  in {A}rtificial {I}ntelligence ({UAI})}, 2009.

\bibitem[Kalisch and B\"{u}hlmann(2007)]{Kalisch2007}
M.~Kalisch and P.~B\"{u}hlmann.
\newblock Estimating high-dimensional directed acyclic graphs with the
  {PC}-algorithm.
\newblock \emph{Journal of Machine Learning Research}, 8:\penalty0 613--636,
  2007.

\bibitem[Kano and Shimizu(2003)]{Kano2003}
Y.~Kano and S.~Shimizu.
\newblock Causal inference using nonnormality.
\newblock In \emph{Proceedings of the International Symposium on Science of
  Modeling, the 30th Anniversary of the Information Criterion}, Tokyo, Japan,
  2003.

\bibitem[Koller and Friedman(2009)]{Koller2009}
D.~Koller and N.~Friedman.
\newblock \emph{Probabilistic Graphical Models: Principles and Techniques}.
\newblock MIT Press, 2009.

\bibitem[Lauritzen(1996)]{Lauritzen1996}
S.~Lauritzen.
\newblock \emph{Graphical Models}.
\newblock Oxford University Press, New York, 1996.

\bibitem[Meek(1997)]{Meek1997}
C.~Meek.
\newblock \emph{Graphical models: selecting causal and statistical models}.
\newblock PhD thesis, Carnegie Mellon University, Pittsburgh, PA, 1997.

\bibitem[Mooij and Janzing(2010)]{CauseEffectPairs}
J.~M. Mooij and D.~Janzing.
\newblock Distinguishing between cause and effect.
\newblock \emph{Journal of Machine Learning Research W\&CP}, 6:\penalty0
  147--156, 2010.

\bibitem[Mooij et~al.(2009)Mooij, Janzing, Peters, and
  Sch\"{o}lkopf]{Mooij2009}
J.~M. Mooij, D.~Janzing, J.~Peters, and B.~Sch\"{o}lkopf.
\newblock Regression by dependence minimization and its application to causal
  inference.
\newblock In \emph{Proceedings of the 26th International Conference on Machine
  Learning ({ICML})}, 2009.

\bibitem[Mooij et~al.(2011)Mooij, Janzing, Heskes, and
  Sch{\"o}lkopf]{Mooij_et_al_NIPS_11}
J.~M. Mooij, D.~Janzing, T.~Heskes, and B.~Sch{\"o}lkopf.
\newblock On causal discovery with cyclic additive noise models.
\newblock In \emph{{A}dvances in {N}eural {I}nformation {P}rocessing {S}ystems
  24 ({NIPS})}, 2011.

\bibitem[Mooij et~al.(2013)Mooij, Janzing, and Sch{\"o}lkopf]{MooijJS2013}
J.~M. Mooij, D.~Janzing, and B.~Sch{\"o}lkopf.
\newblock From ordinary differential equations to structural causal models: the
  deterministic case.
\newblock In \emph{Proceedings of the 29th Conference Annual Conference on
  Uncertainty in Artificial Intelligence ({UAI})}, 2013.

\bibitem[Nowzohour and B\"uhlmann(2013)]{Nowzohour2013}
C.~Nowzohour and P.~B\"uhlmann.
\newblock Score-based causal learning in additive noise models.
\newblock \emph{ArXiv e-prints (1311.6359)}, 2013.

\bibitem[{OEIS Foundation Inc.}(2011)]{OEIS}
{OEIS Foundation Inc.}
\newblock The on-line encyclopedia of integer sequences.
\newblock \url{http://oeis.org/A003024}, 2011.

\bibitem[Pearl(2009)]{Pearl2009}
J.~Pearl.
\newblock \emph{Causality: Models, reasoning, and inference}.
\newblock Cambridge University Press, 2nd edition, 2009.

\bibitem[Peters(2008)]{PetersDiploma}
J.~Peters.
\newblock Asymmetries of time series under inverting their direction.
\newblock Diploma Thesis, University of Heidelberg, 2008.
\newblock \url{http://stat.ethz.ch/people/jopeters}.

\bibitem[Peters(2012)]{PetersThesis}
J.~Peters.
\newblock \emph{Restricted structural equation models for causal inference}.
\newblock PhD thesis, ETH Zurich and MPI for Intelligent Systems, 2012.
\newblock \url{http://dx.doi.org/10.3929/ethz-a-007597940}.

\bibitem[Peters(2014)]{Petersinters}
J.~Peters.
\newblock On the intersection property of conditional independence and its
  application to causal discovery.
\newblock \emph{ArXiv e-prints (1403.0408)}, 2014.

\bibitem[Peters and B\"uhlmann(2013)]{Peters2013sid}
J.~Peters and P.~B\"uhlmann.
\newblock Structural intervention distance ({SID}) for evaluating causal
  graphs.
\newblock \emph{ArXiv e-prints (1306.1043)}, 2013.

\bibitem[Peters and B{\"u}hlmann(2014)]{Peters2014biom}
J.~Peters and P.~B{\"u}hlmann.
\newblock Identifiability of {G}aussian structural equation models with equal
  error variances.
\newblock \emph{Biometrika}, 101:\penalty0 219--228, 2014.

\bibitem[Peters et~al.(2011{\natexlab{a}})Peters, Janzing, and
  Sch\"{o}lkopf]{Peters2011a}
J.~Peters, D.~Janzing, and B.~Sch\"{o}lkopf.
\newblock Causal inference on discrete data using additive noise models.
\newblock \emph{IEEE Transactions on Pattern Analysis and Machine
  Intelligence}, 33:\penalty0 2436--2450, 2011{\natexlab{a}}.

\bibitem[Peters et~al.(2011{\natexlab{b}})Peters, Mooij, Janzing, and
  Sch\"{o}lkopf]{Peters2011b}
J.~Peters, J.~M. Mooij, D.~Janzing, and B.~Sch\"{o}lkopf.
\newblock Identifiability of causal graphs using functional models.
\newblock In \emph{Proceedings of the 27th Annual Conference on {U}ncertainty
  in {A}rtificial {I}ntelligence ({UAI})}, 2011{\natexlab{b}}.

\bibitem[Ramsey et~al.(2006)Ramsey, Zhang, and Spirtes]{Ramsey2006}
J.~Ramsey, J.~Zhang, and P.~Spirtes.
\newblock Adjacency-faithfulness and conservative causal inference.
\newblock In \emph{Proceedings of the 22nd Annual Conference on {U}ncertainty
  in {A}rtificial {I}ntelligence ({UAI})}, 2006.

\bibitem[Rasmussen and Nickisch(2010)]{RasmussenNickisch10}
C.~E. Rasmussen and H.~Nickisch.
\newblock {Gaussian Processes for Machine Learning (GPML) Toolbox}.
\newblock \emph{Journal of Machine Learning Research}, 11:\penalty0 3011--3015,
  2010.

\bibitem[Richardson and Spirtes(2002)]{Richardson2002}
T.~Richardson and P.~Spirtes.
\newblock Ancestral graph {M}arkov models.
\newblock \emph{Annals of Statistics}, 30\penalty0 (4):\penalty0 962--1030,
  2002.

\bibitem[Sch{\"o}lkopf et~al.(2012)Sch{\"o}lkopf, Janzing, Peters, Sgouritsa,
  Zhang, and Mooij]{ScholkopfJPSZMJ2012}
B.~Sch{\"o}lkopf, D.~Janzing, J.~Peters, E.~Sgouritsa, K.~Zhang, and J.~M.
  Mooij.
\newblock On causal and anticausal learning.
\newblock In \emph{Proceedings of the 29th International Conference on Machine
  Learning ({ICML})}, 2012.

\bibitem[Shimizu et~al.(2006)Shimizu, Hoyer, Hyv\"{a}rinen, and
  Kerminen]{Shimizu2006}
S.~Shimizu, P.~Hoyer, A.~Hyv\"{a}rinen, and A.~J. Kerminen.
\newblock A linear non-{G}aussian acyclic model for causal discovery.
\newblock \emph{Journal of Machine Learning Research}, 7:\penalty0 2003--2030,
  2006.

\bibitem[Shimizu et~al.(2011)Shimizu, Inazumi, Sogawa, Hyv\"{a}rinen, Kawahara,
  Washio, Hoyer, and Bollen]{Shimizu2011}
S.~Shimizu, T.~Inazumi, Y.~Sogawa, A.~Hyv\"{a}rinen, Y.~Kawahara, T.~Washio,
  P.~Hoyer, and K.~Bollen.
\newblock Direct{LiNGAM}: A direct method for learning a linear non-{G}aussian
  structural equation model.
\newblock \emph{Journal of Machine Learning Research}, 12:\penalty0 1225--1248,
  2011.

\bibitem[Silva and Ghahramani(2009)]{SilvaGhahramani2009}
R.~Silva and Z.~Ghahramani.
\newblock The hidden life of latent variables: {B}ayesian learning with mixed
  graph models.
\newblock \emph{Journal of Machine Learning Research}, 10:\penalty0 1187--1238,
  2009.

\bibitem[Skitovi{\v{c}}(1954)]{skitovich}
V.P. Skitovi{\v{c}}.
\newblock Linear forms in independent random variables and the normal
  distribution law (in {R}ussian).
\newblock \emph{Izvestiia AN SSSR, Ser. Matem.}, 18:\penalty0 185--200, 1954.

\bibitem[Skitovi{\v{c}}(1962)]{skitovich_trans}
V.P. Skitovi{\v{c}}.
\newblock Linear combinations of independent random variables and the normal
  distribution law.
\newblock \emph{Select. Transl. Math. Stat. Probab.}, 2:\penalty0 211--228,
  1962.

\bibitem[Spirtes et~al.(2000)Spirtes, Glymour, and Scheines]{Spirtes2000}
P.~Spirtes, C.~Glymour, and R.~Scheines.
\newblock \emph{Causation, Prediction, and Search}.
\newblock MIT Press, 2nd edition, 2000.

\bibitem[Tamada et~al.(2011{\natexlab{a}})Tamada, Imoto, Araki, Nagasaki,
  Print, Charnock-Jones, and Miyano]{Tamada2011b}
Y.~Tamada, S.~Imoto, H.~Araki, M.~Nagasaki, C.~G. Print, S.~D. Charnock-Jones,
  and S.~Miyano.
\newblock Estimating genome-wide gene networks using nonparametric {B}ayesian
  network models on massively parallel computers.
\newblock \emph{IEEE/ACM Transactions on Computational Biology and
  Bioinformatics}, 8\penalty0 (3):\penalty0 683--697, 2011{\natexlab{a}}.

\bibitem[Tamada et~al.(2011{\natexlab{b}})Tamada, Imoto, and
  Miyano]{Tamada2011}
Y.~Tamada, S.~Imoto, and S.~Miyano.
\newblock Parallel algorithm for learning optimal {B}ayesian network structure.
\newblock \emph{Journal of Machine Learning Research}, 12:\penalty0 2437--2459,
  2011{\natexlab{b}}.

\bibitem[Teyssier and Koller(2005)]{teykol05}
M.~Teyssier and D.~Koller.
\newblock Ordering-based search: a simple and effective algorithm for learning
  {B}ayesian networks.
\newblock In \emph{Proceedings of the 21st Conference on Uncertainty in
  Artificial Intelligence (UAI)}, 2005.

\bibitem[Tsamardinos et~al.(2006)Tsamardinos, Brown, and
  Aliferis]{Tsamardinos2006}
I.~Tsamardinos, L.~E. Brown, and C.~F. Aliferis.
\newblock The max-min hill-climbing {B}ayesian network structure learning
  algorithm.
\newblock \emph{Machine Learning}, 65\penalty0 (1):\penalty0 31--78, 2006.

\bibitem[Uhler et~al.(2013)Uhler, Raskutti, B\"uhlmann, and Yu]{Uhler2013}
C.~Uhler, G.~Raskutti, P.~B\"uhlmann, and B.~Yu.
\newblock Geometry of the faithfulness assumption in causal inference.
\newblock \emph{Annals of Statistics}, 41\penalty0 (2):\penalty0 436--463,
  2013.

\bibitem[Verma and Pearl(1991)]{Verma1991}
T.~Verma and J.~Pearl.
\newblock Equivalence and synthesis of causal models.
\newblock In \emph{Proceedings of the 6th Annual Conference on {U}ncertainty in
  {A}rtificial {I}ntelligence ({UAI})}, 1991.

\bibitem[Wright(1921)]{Wright1921}
S.~Wright.
\newblock Correlation and causation.
\newblock \emph{Journal of Agricultural Research}, 20:\penalty0 557--585, 1921.

\bibitem[Zhang and Spirtes(2003)]{Zhang2003}
J.~Zhang and P.~Spirtes.
\newblock Strong faithfulness and uniform consistency in causal inference.
\newblock In \emph{Proceedings of the 19th Annual Conference on {U}ncertainty
  in {A}rtificial {I}ntelligence ({UAI})}, 2003.

\bibitem[Zhang and Spirtes(2008)]{Zhang2008}
J.~Zhang and P.~Spirtes.
\newblock Detection of unfaithfulness and robust causal inference.
\newblock \emph{Minds and Machines}, 18:\penalty0 239--271, 2008.

\bibitem[Zhang and Hyv\"{a}rinen(2009)]{Zhang2009}
K.~Zhang and A.~Hyv\"{a}rinen.
\newblock On the identifiability of the post-nonlinear causal model.
\newblock In \emph{Proceedings of the 25th Annual Conference on {U}ncertainty
  in {A}rtificial {I}ntelligence ({UAI})}, 2009.

\bibitem[Zhang et~al.(2011)Zhang, Peters, Janzing, and
  Sch{\"o}lkopf]{Zhang2011}
K.~Zhang, J.~Peters, D.~Janzing, and B.~Sch{\"o}lkopf.
\newblock Kernel-based conditional independence test and application in causal
  discovery.
\newblock In \emph{Proceedings of the 27th Annual Conference on {U}ncertainty
  in {A}rtificial {I}ntelligence ({UAI})}, 2011.

\end{thebibliography}
\end{document}